\pdfoutput=1
\documentclass{article}

\usepackage{microtype}
\usepackage{graphicx}
\usepackage{subcaption}
\usepackage{booktabs}
\usepackage{multirow}
\usepackage{hyperref}
\usepackage{tikz}
\usetikzlibrary{positioning,arrows.meta,shapes.geometric}
\usepackage{algorithm}
\usepackage{algorithmic}

\usepackage[accepted]{icml2026}
\makeatletter
\renewcommand{\Notice@String}{Accepted to the TAIGR Workshop at ICML~2026.
Copyright 2026 by the author.}
\makeatother
\hypersetup{pdfsubject={Accepted to the TAIGR Workshop at ICML 2026}}

\usepackage{amsmath}
\usepackage{amssymb}
\usepackage{mathtools}
\usepackage{amsthm}
\usepackage[capitalize,noabbrev]{cleveref}

\theoremstyle{plain}
\newtheorem{theorem}{Theorem}[section]

\newtheorem{lemma}[theorem]{Lemma}
\theoremstyle{definition}

\theoremstyle{remark}

\newcommand{\todo}[1]{}  

\icmltitlerunning{Evidence-Grounded Verified Agentic Reasoning}

\begin{document}

\twocolumn[
\icmltitle{Evidence-Grounded Verified Agentic Reasoning:\\
A Path Toward Eliminating LLM Hallucination in Empirical Inference via Tool-Attested Kernel Proofs}

\icmlsetsymbol{equal}{*}

\begin{icmlauthorlist}
\icmlauthor{Junyu Ren}{uchicago}
\end{icmlauthorlist}

\icmlaffiliation{uchicago}{Committee on Computational and Applied Mathematics, University of Chicago, Chicago, IL, USA}

\icmlcorrespondingauthor{Junyu Ren}{junyuren@uchicago.edu}

\icmlkeywords{verifiable AI, hallucination, formal verification, AI governance, Lean 4, source-faithfulness}

\vskip 0.3in
]

\printAffiliationsAndNotice{}

\begin{abstract}
Tool access alone does not make LLM empirical reasoning
governable: accepted outputs need not descend from attested
evidence, and accepted deductions need not hold up under formal
scrutiny. We present EG-VAR (Evidence-Grounded Verified Agentic
Reasoning), a Lean~4-based tool-calling architecture in which the
Lean kernel is the sole minter of \textsc{Verified} claims via
tool-attestation axioms and declared source lifts. Every verified
output structurally descends from an attested tool call
(\Cref{thm:novh}) and a kernel-checked chain of valid inference
(\Cref{thm:noded}); residual outputs are honest \textsc{Abstain}
with a replayable audit trail. On a subcollection of TableBench
numerical reasoning ($n{=}120$), EG-VAR attains 120/120 versus a
95\% same-tool baseline; on counterfactual stress tests
(5~domains $\times$ 2~models), EG-VAR stays 100\%
source-faithful while same-tool drops to 80--90\% (no-tool
50--80\%). With the LLM as deployment-time formalizer, residual
semantic-formalization error is 3.3\% on Sonnet and 1.7\% on Opus.
We position EG-VAR as a technical-governance interface for
high-stakes empirical claims: a formal sidecar makes the target
proposition, source scope, evidence boundary, proof obligation,
and abstention condition auditable, eliminating unsupported
\textsc{Verified} outputs today while turning formalization errors,
lift and source-authority disputes, ambiguities, and abstentions
into explicit audit targets. Over time, typed sidecars in datasets, APIs, public
records, and AI-generated documents can amortize this
formalization burden into reusable infrastructure.
\end{abstract}

\section{Introduction}
\label{sec:intro}

LLMs hallucinate on empirical facts when making inference,
fabricating or reproducing outdated ones~\citep{huang2025hallucination}. Tools
and external memories were introduced as a mitigation, letting the
LLM consult an authoritative source rather than rely on its
parametric prior. But tool access alone does not yield verifiable
empirical claims: frontier models exhibit strong confirmation
bias, dismissing retrieved evidence that conflicts with parametric
memory~\citep{xie2024adaptive}, and even state-of-the-art models
struggle to stay faithful to counterfactual
context~\citep{faitheval2025}; our own counterfactual stress tests
on tables reproduce the pattern at $\sim$12--15\%, robust to an
explicit ``according to the table'' source-anchoring cue
(\Cref{sec:eval:tier15}).

Tool access leaves two structural failures uncovered: accepted
outputs need not descend from attested evidence, and accepted
deductions need not hold up under formal scrutiny. EG-VAR targets
each with a structural contract: every \textsc{Verified} output
must descend from an attested tool call (\Cref{thm:novh}), and
every accepted proof step must type-check in the Lean kernel under
the declared axiom set (\Cref{thm:noded}). \textsc{Verified} means
attested evidence plus kernel-checked derivation; everything else
surfaces as honest \textsc{Abstain} with a replayable audit trail.
We call this combination a \emph{verified empirical claim}: a
claim any auditor can later check without re-running the model.

Concretely, given a claim like ``Brazil scored on average $x$
more goals in the 2018 opener than Italy'' against a spreadsheet,
EG-VAR compiles the claim with indeterminate answer into a typed
goal in Lean~4. An LLM uses tool calls to fetch attested values
and compute statistics from the table, reasons and infers an
answer value derived from the attested evidence, and proposes a
Lean tactic encoding this derivation as a proof; the Lean kernel
type-checks the tactic. Accepted:
publish the answer with a machine-checkable proof; rejected:
abstain. A third party re-checks offline by re-running Lean; the
LLM need not be re-run.

\paragraph{Our contribution.} We present \textbf{EG-VAR}
(Evidence-Grounded Verified Agentic Reasoning), a source-general
architecture for empirical claims grounded in attested external
data, with tables as the first end-to-end instantiation. To our
knowledge this is the first proof-assistant-verified architecture
in Lean~4 for this setting. The LLM is one
untrusted layer in a four-layer stack: a deterministic
\emph{tool layer} (L1, attested storage queries), a per-source
\emph{formalization layer} (L2, audited semantic lifts from
storage to world facts), a \emph{Lean~4 kernel} that mints
\emph{Verified} world claims solely via attestation-derived rules, and
the \emph{solver LLM} that chooses the answer witness, substitutes
it into the claim's dependent-sum goal, and proposes a proof the
kernel checks. The key
architectural commitment is that \texttt{mkVerified} (the sole
constructor for \texttt{Evidence Verified~$w$}) requires an
\texttt{Attested\_T} from the runtime, so every verified output
structurally descends from a tool call. The certified answer is
extracted as the witness of a dependent sum (\citealp{martin1984type})
in the Curry--Howard tradition~\citep{howard1980formulae}: a question
``which entity has property~$P$?''~becomes a goal of type
$\Sigma\,(e:\mathrm{String})\,(v:\mathrm{Value}),
\mathrm{Evidence~Verified}\,(\ldots\,e\,v)$, and the surface answer is
the first projection of the kernel-checked witness. Verified outputs
come with replayable proof traces; unverified claims abstain rather
than guess. Unlike symbolic factual verifiers, execution-attestation
systems, and math theorem provers (\Cref{sec:related}), EG-VAR
combines proof-assistant checking with attested empirical grounding.

\paragraph{Why a proof assistant for empirical grounding?} The
standard objection is that proof assistants are overkill for
empirical reasoning over external sources. We disagree: existing
LLM-driven theorem-proving systems (AlphaProof~\citep{alphaproof2024},
Aristotle~\citep{aristotle2025},
DeepSeek-Prover~\citep{deepseek_prover_v15_2024},
Hilbert~\citep{hilbert_apple_2025}) already show that the substrate
is not the bottleneck; the bottleneck is audited per-source semantic
lifting, which our formalization protocol addresses
(\Cref{sec:approach:pertable}; instantiated here on tables). This
cost is amortized: each per-source lift is audited once and reused
across all subsequent claims grounded in that source, so
formalization never enters the per-query inference loop.

\paragraph{Scope of the empirical work.} All empirical results in
this paper use the tabular instantiation and evaluate on
TableBench~\citep{tablebench}. This is the first end-to-end demonstration of the
L1/L2/lift pattern on a single substrate, not the conceptual
boundary of the framework; cross-substrate validation (SQL views,
typed APIs, knowledge graphs) is future work.

\paragraph{Empirical findings.} On TableBench unambiguous claims
(Tier~1 gold-goal benchmark, $n{=}120$), EG-VAR attains 120/120
(100\%); the same-tool
same-model baseline plateaus at 114/120 (95.0\%) and the table-only
single-shot baseline at 107/120 (89.2\%). On a 5-domain counterfactual
stress panel (Tier~1.5 binary-compare and direction claims, $n{=}20$
per panel, two models, two injection regimes), EG-VAR remains 100\% source-faithful in
every cell; on Sonnet the same-tool baseline drops to 80--90\% and
the no-tool baseline to 50--80\%, depending on injection magnitude. Under full end-to-end
deployment with the LLM itself as the formalizer (Tier~2, $n{=}120$),
the pure formalizer-error rate is 3.3\% on Sonnet and 1.7\%
on Opus, with the remaining disagreements split into logged
ambiguities (4.2\%) and surfaced benchmark-artifact issues (1.7\%)
that are auditable rather than silent (\Cref{sec:eval:tier2}). The
discriminator is robust to model strength, source cue, and flip
magnitude.

\paragraph{A path toward eliminating LLM hallucination.} The
title claim has a precise structural meaning
(\Cref{thm:novh}): no \textsc{Verified} output can appear without
an attested tool leaf, under any formalizer. Semantic
mis-formalization remains non-zero at current LLM capability
(3.3\% on Sonnet, 1.7\% on Opus), but this accuracy axis is
addressable by fine-tuning on the curator-grade (claim, gold goal)
pairs the paper releases (code supplement in
\Cref{app:repro}; curator pipeline summary in
\Cref{sec:approach:provenance} with full proposer/critic/kernel
enumerate in \Cref{app:provenance}; case studies in
\Cref{app:formalizer-case-studies}) without weakening the
structural guarantee.

\paragraph{Formalization flywheel.} The long-run bet is not that
users will manually write Lean. It is that empirical communication
can gradually acquire \emph{formal sidecars}. In the near term,
EG-VAR's curator corpus supplies supervised and verifier-guided
training data for better NL-to-formal translation. In the medium
term, many sources already move along a substrate spectrum from
CSV to SQL views, OpenAPI schemas, knowledge graphs, and typed
registries, reducing the per-source lift burden
(\Cref{app:scaling}). In the long term, LLM-generated answers and
documents can themselves carry formal copies of key claims
alongside prose. Data media already carry informal instructions of use (schemas,
READMEs); EG-VAR just requires them formal. When upstream
providers themselves run EG-VAR-like pipelines, formalization
becomes mechanized rather than handcrafted, and data arrives
already bearing kernel-checked \textsc{Verified} instructions.
The formalization burden moves out of the per-query inference loop
into reusable source-side infrastructure; the kernel boundary
continues to enforce that only attested claims are labeled
\textsc{Verified}.

\paragraph{Governance translation.} The two safety theorems
(\Cref{sec:approach:safety}), combined with the no-upcast grade
discipline, yield three structural governance properties: (i) tool-
attestation-as-axiom enables independent audit of every verified
output; (ii) grade discipline makes uncertainty typed and queryable
rather than a hidden confidence number; (iii) honest abstention (a
kernel rejection becomes the user-visible outcome) is a measurable
property, not a learned behaviour. We unpack each in
\Cref{sec:discuss:gov}.

\section{Related Work}
\label{sec:related}

We compare EG-VAR against four buckets of prior work
(\Cref{tab:related}): empirical studies of how LLMs trade off
parametric priors against external context; symbolic factual
verifiers; LLM-driven theorem provers; and execution-attestation
systems. Each contributes part of the trust pipeline, but none
combines a proof-assistant kernel, tool-attestation as a typed axiom,
and honest abstention as a structural property for empirical claims.

\begin{table*}[!tbp]
\centering\scriptsize
\setlength{\tabcolsep}{3pt}
\begin{tabular}{@{}lllll@{}}
\toprule
System & Substrate & Evidence input & Checker & Residual risk \\
\midrule
TabVer~\citep{tabver2024} & Nat.\ logic + arith. & Table cells & LLM + NatLog & Mapping error \\
ProoFVer~\citep{proofver2022} & Natural logic & Premises & Seq2seq + NatLog & Mapping error \\
FoVer~\citep{fover2025} & FOL + Z3 & Indirect (CoT) & SMT & Consistency $\not\Rightarrow$ grounded \\
AWS Bedrock~\citep{aws_automated_reasoning} & Undisclosed & Policy text & Proprietary & Undisclosed \\
\midrule
AlphaProof~\citep{alphaproof2024} & Lean~4 (math) & None & Lean kernel & --- \\
DeepSeek-Prover~\citep{deepseek_prover_v15_2024} & Lean~4 (math) & None & Lean kernel & --- \\
\midrule
Attestable Audits~\citep{schnabl2025attestable} & TEE crypto & Exec.\ bytes & TEE attestation & Wrong-but-attested \\
Proofs of Autonomy~\citep{grigor2025proofs} & MPC-TLS (Web Proofs) & Run provenance & Crypto proof & Wrong-but-attested \\
\midrule
\textbf{EG-VAR (ours)} & \textbf{Lean~4 + tool attest.} & \textbf{Attested L1$\to$L2 lifts} & \textbf{Lean kernel} & \textbf{Honest abstention} \\
\bottomrule
\end{tabular}
\caption{Prior work by substrate, evidence input (what the checker
consumes to verify empirical claims), checker, and residual risk
under EG-VAR's threat model.
EG-VAR is the only row combining a proof-assistant kernel with
tool-attested grounding and honest abstention on rejection.}
\label{tab:related}
\end{table*}

The closest empirical prior is ClashEval~\citep{clasheval2024}
(prior-vs-context conflict on frontier models across six domains;
our counterfactual stress tests in \Cref{sec:eval:tier15}
reproduce its prior-retention ranking). Adjacent substrates:
symbolic factual verification~\citep{tabver2024,proofver2022,fover2025}
with learned proof-construction components and mapping-error failure modes;
post-hoc fact-checking~\citep{factool2023,rarr2023,safe2024,cove2023};
commercial policy guardrails~\citep{aws_automated_reasoning};
LLM-for-math theorem proving~\citep{alphaproof2024,deepseek_prover_v15_2024,aristotle2025,hilbert_apple_2025}
(same trust model, self-contained claims, no external grounding);
autoformalization~\citep{dtv2024}; long-lineage CS
verification~\citep{raggi2022representation,bundy1991proof,clarke2003cegar,lahiri2013differential};
and hardware-attested AI
verification~\citep{schnabl2025attestable,vet_your_agent_2025,grigor2025proofs}
(attests execution provenance, not claim truth; complementary to EG-VAR).
Detailed treatment of each thread, including the axes EG-VAR
extends over ClashEval and why self-contained-math provers do not
cover empirical claims, is in \Cref{app:related-extended}.

\section{EG-VAR}
\label{sec:approach}

\subsection{Four-layer stack and trust ledger}
\label{sec:approach:arch}

EG-VAR partitions the agentic-AI loop into four layers with a
strict trust ledger (\Cref{fig:fourlayer}; Threat-model summary
box below).

\textbf{L1, Tool layer (trusted, deterministic).} The tool layer
exposes a fixed set of queries against a source: cell lookups,
filtered aggregates, and selection operators on tables, but the
pattern is source-agnostic. Each query produces an attested payload:
the deterministic answer the tool computed, packaged with a runtime
\texttt{Attested\_T} witness that records query, source identity, and
result. The runtime is the only entity that can produce
\texttt{Attested\_T} witnesses.

\textbf{L2, Per-source formalization (audited offline).} Tools attest
\emph{storage facts} (e.g., ``cell at row $r$, column $c$, has
value~$v$''); claims live at the world-ontology level
(\texttt{HasProperty}, \texttt{ArgmaxWhere}, \texttt{SumWhere}, ...).
The mapping from L1 storage facts to L2 world facts is a curator's
interpretive commitment (in this paper, an audited
proposer--critic--kernel pipeline; \Cref{app:provenance}), encoded
as a per-source list of typed \emph{lifts} that the kernel imports
as axioms (\Cref{sec:approach:pertable}). The L1/L2 split keeps the formal
language describing the world rather than the storage vehicle, the
same separation \citet{codd1970relational} drew between a
relational schema and the predicates it serves.

\textbf{L3, Lean~4 kernel (trusted, formal).} Lean type-checks every
proof step. The kernel is the only entity authorized to mint
\texttt{Evidence Verified~$w$} for any world-claim $w$. There is
exactly one mint rule, \texttt{mkVerified}, which requires an
\texttt{Attested\_T} hypothesis (\Cref{sec:approach:mkverified}).

\textbf{L4, Solver LLM (untrusted, stochastic).} The LLM receives a
typed goal and proposes proof tactics. The kernel rejects invalid
tactics and the rejection messages drive revision. The solver cannot
synthesize \texttt{Attested\_T} terms; it can only consume those
introduced by tool calls during the loop.

\paragraph{Trust ledger.} The kernel, the L1 vocabulary, the per-source
lifts, and the tool adapter are auditable artifacts. The solver is
type-checked, not trusted: an arbitrarily wrong solver can never produce
a verified-but-unsupported output. This is a \emph{structural}
guarantee, not a statistical one. Each verified output is accompanied
by the generated Lean proof file, a self-contained replay artifact
that any third party with the kernel and the per-source lifts can
re-typecheck.

\begin{figure}[tbp]
\centering
\fbox{\begin{minipage}{0.92\columnwidth}
\small
\textbf{Threat model (summary).} Safety theorems
(\Cref{thm:novh,thm:noded}) assume:
\textbf{Trusted:} Lean~4 kernel, EG-VAR prelude, audited
per-source lifts $\Lambda(s)$, deterministic Python runtime.
\textbf{Untrusted:} LLM analyst/formalizer/solver; they may
hallucinate but cannot introduce axioms.
\textbf{Out of scope:} (i)~semantic faithfulness of $\tau$ to
the NL claim; (ii)~a semantically wrong audited lift;
(iii)~source-data corruption that passes the runtime identity-cell
validator. Items (i)--(ii) detailed in \Cref{sec:discuss:limits}.
\end{minipage}}
\end{figure}

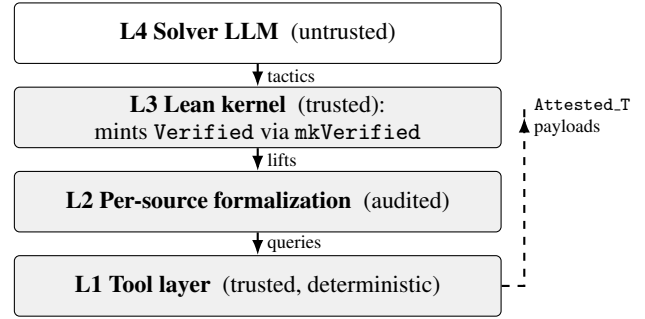
\begin{figure}[t]
\centering
\begin{tikzpicture}[
  node distance=3mm,
  font=\footnotesize,
  layer/.style={draw, rounded corners=2pt, minimum width=0.78\columnwidth,
    minimum height=8mm, align=center, inner sep=3pt},
  trusted/.style={layer, fill=gray!12},
  untrusted/.style={layer, fill=white},
  flow/.style={-{Latex[length=2mm]}, thick},
]
\node[untrusted] (l4) {\textbf{L4 Solver LLM}~~(untrusted)};
\node[trusted, below=of l4] (l3) {\textbf{L3 Lean kernel}~~(trusted): \\
  mints \texttt{Verified} via \texttt{mkVerified}};
\node[trusted, below=of l3] (l2) {\textbf{L2 Per-source formalization}~~(audited)};
\node[trusted, below=of l2] (l1) {\textbf{L1 Tool layer}~~(trusted, deterministic)};
\draw[flow] (l4) -- (l3) node[midway, right, font=\scriptsize]{tactics};
\draw[flow] (l3) -- (l2) node[midway, right, font=\scriptsize]{lifts};
\draw[flow] (l2) -- (l1) node[midway, right, font=\scriptsize]{queries};
\draw[flow, dashed] (l1.east) -- ++(3mm,0) |-
  ([xshift=3mm]l3.east)
  node[midway, right, font=\scriptsize, align=left]
  {\texttt{Attested\_T}\\payloads};
\end{tikzpicture}
\caption{EG-VAR four-layer architecture. Solid arrows: control
flow. Dashed arrow: \texttt{Attested\_T} payloads, the sole
load-bearing hypothesis of \texttt{mkVerified}
(\Cref{sec:approach:mkverified}).}
\label{fig:fourlayer}
\end{figure}

\subsection{Tables as formal objects}
\label{sec:approach:tablesasformal}

The deeper move in EG-VAR is on the source side. Empirical
verification fails not because claims are unformalized, but because
the \emph{source's semantics} are implicit in the tool adapter,
hidden in fuzzy retrieval, or assumed in prompt scaffolding. The
source is a co-equal trust artifact: from the kernel's perspective a
table is an object that needs formalization on the same footing as
a claim.

Concretely, three things are formalized per source:
(i) an \textbf{L1 storage vocabulary} (\texttt{CellAt},
\texttt{ColumnContains}, \texttt{RowMatches}, ...) over which tools
emit attested payloads;
(ii) an \textbf{L2 world ontology} (\texttt{HasProperty},
\texttt{ArgmaxWhere}, \texttt{SumWhere}, ...) over which claims are
stated;
(iii) a \textbf{per-source lift catalog} (\Cref{sec:approach:pertable})
encoding the curator's interpretive commitment about which L1 facts
attest which L2 claims under this table's chosen semantics. The
trust artifact a third party must inspect is the lift catalog; tools
carry L1 facts deterministically, the kernel reasons only over L2
evidence. This separates source-side curation from claim-side
reasoning, making both auditable independently.

A dataset admits multiple consistent formalizations (e.g., a
matchup table as rank-2 curried vs rank-1 with composite-domain);
these are type-isomorphic in any Cartesian closed setting and the
kernel admits any consistent one. The curator picks an
interpretation and audits the lift menu against the table;
pluralism is architectural. Full gauge-freedom argument in
\Cref{app:rank}.

We trade internal completeness for soundness, decidability, and
proof economy: L2 predicate equivalences (e.g.\
$\textsf{MeanWhere} \Leftrightarrow
\textsf{SumWhere}/\textsf{CountWhere}$) are localized in per-source
lifts rather than added as kernel lemmas, keeping proof search
decidable and traces inspectable. This mirrors SMT theory
combination, description logics, and decidable FOL fragments;
detailed rationale in \Cref{app:anchors}.

\subsection{Evidence grades and downcast-only}
\label{sec:approach:grades}

Evidence carries an explicit grade
$g \in \{\textsc{Verified}, \textsc{Supported}, \textsc{Plausible},
\textsc{Speculative}\}$ ordered by reliability. The kernel admits
\emph{downcasting} only ($g_1 \succeq g_2$ and $e :
\texttt{Evidence}~g_1~w \Rightarrow \texttt{downcast}~e :
\texttt{Evidence}~g_2~w$). No upcast rule exists, so conclusions
cannot outrank their weakest premise. The only constructor for
\texttt{Evidence Verified~$\_$} is \texttt{mkVerified}, described
next; lower grades arise from explicit downgrades (e.g., staleness,
source-conflict) or LLM-proposed heuristic claims that the calculus
admits at \textsc{Plausible} or \textsc{Speculative} only.

\subsection{The \texttt{mkVerified} rule (sole minter of
\textsc{Verified})}
\label{sec:approach:mkverified}

Informally: verified evidence exists only if the runtime already
supplied an attested payload. The kernel's mint rule is:
\vspace{-4pt}
\begin{small}
\begin{verbatim}
mkVerified :
  forall (T : ToolId) (q : Query_T)
         (p : Payload_T q)
         (att : Attested_T q p) (w : WProp),
  w in claims (interp_T q p) ->
  Evidence Verified w
\end{verbatim}
\end{small}
\vspace{-8pt}

\noindent The hypothesis \texttt{att : Attested\_T q p} is the
load-bearing element. \texttt{Attested\_T} witnesses are introduced
into the proof context only by the runtime in response to a tool call;
the solver cannot synthesize them. \texttt{interp\_T q p} is the
per-tool interpretation function (a small total Lean function) that
maps the attested payload to a finite set of L2 world claims; the
kernel mints \texttt{Verified} for any $w$ in that set. There is no
other path to \texttt{Evidence Verified~$\_$}.

\subsection{Per-source formalization as the curator's interpretive
commitment}
\label{sec:approach:pertable}

The interesting design problem is the L1$\to$L2 bridge. A tool attests
``cell at row~$r$, column~``state'', has value~``alabama'' '' (an L1
storage fact); the claim under verification is ``Alabama has the
highest HIV incidence rate'' (an L2 world fact). The bridge is a
per-source list of typed lifts, audited offline by the source's
curator, that the kernel imports as axioms.

Per-source lifts operationalize the L1$\to$L2 bridge: typed
composition of an attested L1 fact with a audited lift
axiom mints an \texttt{Evidence Verified} witness for the matching
L2 world-property. Absence of a matching lift is a feature: the
solver has no L2 proof path, the kernel rejects, and the system
abstains rather than silently guess. The L1/L2 separation is
source-agnostic (SQL views, knowledge graphs, and sensor streams
admit the same pattern; \Cref{sec:discuss:gov}); a complete worked
example (HIV argmax via state projection, showing the lift,
attested payload, and Sigma-witness extraction) is in
\Cref{app:hiv-argmax}.

\subsection{Operational proof object and safety theorems}
\label{sec:approach:safety}

The safety guarantees are stated about the runtime's operational
\emph{proof object}, not about Lean's type theory alone. A proof
object is the tuple of artifacts that the Python implementation at
\texttt{src/python/runtime/pipeline.py} actually returns from
\texttt{run\_pipeline} as the fields of \texttt{PipelineResult}:
\[
  \mathrm{ProofObject}(s, c)
  \;=\; \langle\, \textit{leanFile},\; \textit{evidence},\;
                 \textit{steps} \,\rangle
\]
where \textit{leanFile} is the generated Lean module text,
\textit{evidence} is the time-ordered list of tool-call payloads
lifted to observation-leaf axioms, and \textit{steps} is the history
of solver actions and kernel-verbatim feedback.
\Cref{fig:trace} gives a concrete end-to-end instance for one
claim. A standalone
replay-audit script
(\texttt{src/python/runtime/verify\_axioms.py}, \Cref{app:repro})
takes \textit{leanFile} as input and recomputes the whitelist check
post-hoc; the script is not invoked in the \texttt{run\_pipeline}
loop. The primary safety gate at generation time is the Lean kernel
invocation \texttt{lake env lean} inside \texttt{pipeline.py}.

\paragraph{Trust zones.} Four roles participate in producing the
proof object. \textbf{Untrusted proposers}: the LLM analyst,
formalizer, and solver; none can introduce new top-level axioms
because their output is constrained to dictionary entries, a proposed
goal type, and a tactic body inserted inside a pre-written Sigma type
declaration. \textbf{Audited runtime-orchestrator}: the deterministic
Python layer that executes tools, emits observation-leaf axioms,
assembles the Lean module, invokes the kernel, and records
\textit{evidence} and \textit{steps}. \textbf{Trusted checker}: the
Lean~4 kernel. \textbf{Trusted domain assumptions}: the fixed EG-VAR
prelude (\texttt{EGVar/Basic.lean}, \texttt{EGVar/TableL1.lean},
\texttt{EGVar/WorldOntology.lean}, and \texttt{EGVar/Tactics.lean})
and the audited per-source lifts $\Lambda(s)$ from
\texttt{data/tier1\_table\_formalizations.json}, committed at
curation time per \Cref{sec:approach:provenance}.

\begin{algorithm}[t]
\caption{\textsc{RuntimeVerify}: operational definition of the
runtime loop. Trust-zone labels annotate each step. The Python
implementation at \texttt{src/python/runtime/pipeline.py} realizes
this algorithm.}
\label{alg:runtime}
\begin{algorithmic}[1]
\STATE {\bfseries Input:} claim $c$, source $s$
\STATE {\bfseries Output:} proof object and verdict
\STATE $\textit{guide} \gets \textsc{Analyst}(s, c)$
  \hfill\textit{// untrusted LLM}
\STATE $\textit{goalType} \gets
  \textsc{Formalizer}(c, \textit{guide}, \textit{dict}, \Lambda(s))$
\STATE $\textit{evidence} \gets [\,]$;\quad
  $\textit{steps} \gets [\,]$
\FOR{$i = 1$ {\bfseries to} $\textit{maxSteps}$}
  \STATE $\textit{action} \gets
    \textsc{Solver}(c, \textit{goalType}, \textit{evidence})$
  \IF{$\textit{action}$ is a tool call $(t, q)$}
    \STATE $\textit{payload} \gets \textsc{Execute}(t, q)$
      \hfill\textit{// audited tool}
    \IF{$\textit{payload}$ succeeded}
      \STATE append $(t, q, \textit{payload},
        \Lambda(s).\textit{lift})$ to $\textit{evidence}$
    \ENDIF
  \ELSIF{$\textit{action}$ is a proof attempt $\textit{tac}$}
    \STATE $\textit{leanFile} \gets
      \textsc{Build}(\textit{goalType}, \Lambda(s),
        \textit{evidence}, \textit{tac})$
    \STATE $(\textit{ok}, \textit{out}) \gets
      \textsc{KernelCheck}(\textit{leanFile})$
      \hfill\textit{// trusted Lean kernel}
    \IF{$\textit{ok}$}
      \STATE {\bfseries return} $(\textit{leanFile}, \textit{evidence},
        \textit{steps})$, \textsc{Verified}
    \ENDIF
  \ENDIF
  \STATE append $(\textit{action}, \textit{feedback})$ to
    $\textit{steps}$
\ENDFOR
\STATE {\bfseries return} $\bot$, \textsc{Abstain}
\end{algorithmic}
\end{algorithm}

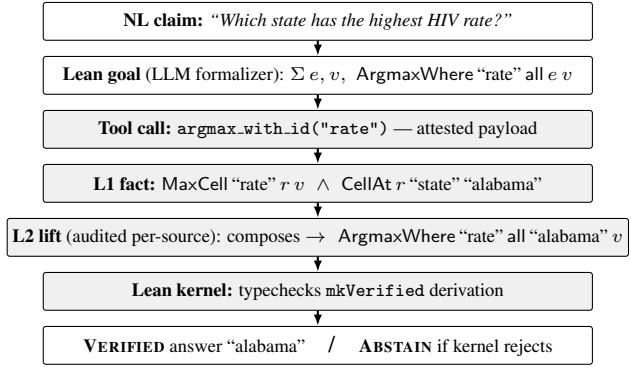
\begin{figure}[t]
\centering
\begin{tikzpicture}[
  node distance=2mm,
  font=\scriptsize,
  box/.style={draw, rounded corners=1pt, minimum width=0.88\columnwidth,
    minimum height=5mm, align=center, inner sep=2pt},
  untrusted/.style={box, fill=white},
  trusted/.style={box, fill=gray!12},
  flow/.style={-{Latex[length=1.5mm]}, thick},
]
\node[untrusted] (nl) {\textbf{NL claim:} \emph{``Which state has the highest HIV rate?''}};
\node[untrusted, below=of nl] (goal) {\textbf{Lean goal} (LLM formalizer): $\Sigma\,e,v,\;
  \mathsf{ArgmaxWhere}\,\text{``rate''}\,\mathsf{all}\,e\,v$};
\node[trusted, below=of goal] (tool) {\textbf{Tool call:} \texttt{argmax\_with\_id("rate")} --- attested payload};
\node[trusted, below=of tool] (L1) {\textbf{L1 fact:}
  $\mathsf{MaxCell}\,\text{``rate''}\,r\,v\;\wedge\;\mathsf{CellAt}\,r\,\text{``state''}\,\text{``alabama''}$};
\node[trusted, below=of L1] (L2) {\textbf{L2 lift} (audited per-source): composes $\to\;\mathsf{ArgmaxWhere}\,\text{``rate''}\,\mathsf{all}\,\text{``alabama''}\,v$};
\node[trusted, below=of L2] (kernel) {\textbf{Lean kernel:} typechecks \texttt{mkVerified} derivation};
\node[untrusted, below=of kernel] (out) {\textbf{\textsc{Verified}} answer ``alabama''
  \quad{\footnotesize{\small /}} \quad \textbf{\textsc{Abstain}} if kernel rejects};
\draw[flow] (nl) -- (goal);
\draw[flow] (goal) -- (tool);
\draw[flow] (tool) -- (L1);
\draw[flow] (L1) -- (L2);
\draw[flow] (L2) -- (kernel);
\draw[flow] (kernel) -- (out);
\end{tikzpicture}
\caption{Worked trace for one claim. Untrusted LLM output (white)
commits to a goal type; trusted runtime (gray) produces the
attested payload, composes it through the audited per-source lift,
and the Lean kernel mints \textsc{Verified} iff the derivation
type-checks. The Sigma-witness of the accepted proof is the
published answer.}
\label{fig:trace}
\end{figure}

\paragraph{Whitelist-compliant module (summary).} We call
\textit{leanFile} \emph{whitelist-compliant for source} $s$ iff
every axiom in the dependency closure of the claim proof belongs
to one of four buckets: the fixed EG-VAR prelude (\textbf{B1}),
standard Lean~4 core axioms (\textbf{B2}), runtime-emitted
observation leaves (\textbf{B3}), or per-source lifts matching
the regex \texttt{lift\_[0-9a-f]+\_.+} drawn from $\Lambda(s)$
(\textbf{B4}); the closure contains no \texttt{sorry}
variant and at least one trust-artifact witness. Full bucket
schemas, the replay-audit mechanics, and guarantee-boundary
caveats are in \Cref{app:whitelist}.

\begin{theorem}[No unsupported \textsc{Verified} outputs]
\label{thm:novh}
If \textsc{RuntimeVerify}($s, c$) returns $(\textit{leanFile},
\textit{evidence}, \textit{steps})$ with verdict \textsc{Verified},
and \textit{leanFile} is whitelist-compliant for $s$, then every
runtime-emitted observation-leaf axiom in the dependency closure of
the claim proof is backed by a prior tool execution recorded in
\textit{evidence}, and every remaining axiom dependency lies in the
fixed EG-VAR prelude, in the standard Lean~4 core axioms, or in
$\Lambda(s)$.
\end{theorem}

\begin{theorem}[No deductive errors]
\label{thm:noded}
Every inference step that the system accepts as part of a
\textsc{Verified} output type-checks in Lean~4's kernel under the
declared axiom set.
\end{theorem}

\noindent Full proofs (checker-fidelity, emission-discipline, and
reserved-name lemmas; complete proofs of \Cref{thm:novh,thm:noded})
plus the emission-site table are in \Cref{app:safety-proofs,app:whitelist}.
The kernel call inside \texttt{run\_pipeline}
(\texttt{lake env lean}) is the primary safety gate at generation
time; the standalone \texttt{verify\_axioms.py} is a post-hoc
replay audit a third-party can re-run offline
(\Cref{app:repro,app:whitelist}). Worked examples of attested
proofs and per-source lifts are in
\Cref{app:worked-example,app:failure}.

\paragraph{Formalization provenance.}
\label{sec:approach:provenance}
Tier~1 and Tier~1.5 load fixture-committed gold types $\tau^{\mathrm{gold}}$
and audited per-source lifts $\Lambda(s)$ from a frozen fixture
(proposer--critic--kernel pipeline, no formalizer invoked at
evaluation); Tier~2 invokes an LLM formalizer with access to
$\Lambda(s)$ but no gold type. Full provenance (subagent pipeline,
human meta-supervision scope, curator yield, collusion firewall)
is in \Cref{app:provenance}.

\section{Empirical Evaluation}
\label{sec:eval}

We report three evaluations. The first two load fixture-committed
gold goal types, so they isolate the
kernelized trust architecture under ideal formalization:
\Cref{sec:eval:tier1} (Tier~1) establishes competitiveness with
same-tool baselines on unambiguous claims;
\Cref{sec:eval:tier15} (Tier~1.5) establishes
source-faithfulness under counterfactual prior conflict, where the
same-tool baseline on Sonnet drops to 80--90\% (no-tool 50--80\%)
while EG-VAR remains 100\% across both flip-magnitude regimes and
both model strengths. The third evaluation
(\Cref{sec:eval:tier2}, Tier~2) removes the gold-fixture bypass
and tests the deployment path where the LLM itself proposes the
formal goal.

\paragraph{Two axes.} We separate \emph{formal safety} (the
\textsc{Verified} label is kernel-accepted against the committed
trust artifacts; \Cref{thm:novh}) from \emph{semantic
faithfulness} (the formalized goal $\tau$ matches the
natural-language claim; measured by \texttt{answer\_match} against
frozen gold answers). Tiers~1 and~1.5 load
$\tau=\tau^{\mathrm{gold}}$ from the frozen fixture, so semantic
faithfulness holds by construction; Tier~2 invokes the LLM
formalizer, and semantic faithfulness becomes the primary
empirical question.

\paragraph{Setup.} Tier~1 and Tier~1.5 share four \emph{rungs}
(\Cref{tab:rungs}):
Table-only (single-shot, no tools);
Tools-open (LLM with full runtime tool surface, no per-source menu
enforcement);
Tools-curated (per-source menu enforced; same tool surface as
EG-VAR);
EG-VAR (kernel-checked, fixed-goal mode with the curator's typed
goal). All rungs use the same model, prompt structure,
and answer judge (\texttt{answers\_match}); EG-VAR additionally requires
that the kernel mints a \texttt{Verified} witness whose Sigma extract
agrees with the gold answer. We evaluate two models:
\texttt{claude-sonnet-4-6} and \texttt{claude-haiku-4-5} (a
stronger and a weaker model from the same family). All runs use temperature~0; EG-VAR is deterministic at
temperature~0 and we report 1-rep smoke evaluations as canonical for
EG-VAR with corroborating multi-rep partial runs (provenance manifest in
\Cref{app:repro}).

\begin{table}[t]
\centering\small
\setlength{\tabcolsep}{4pt}
\begin{tabular}{@{}lll@{}}
\toprule
Rung & Tools & Per-source artifact \\
\midrule
Table-only & none & --- (table only) \\
Tools-open & full surface & none enforced \\
Tools-curated & menu-enforced & per-source menu \\
EG-VAR & menu-enforced & menu $+$ Lean goal $+$ lifts \\
\bottomrule
\end{tabular}
\caption{Four-rung ladder isolating tool contribution
(Table-only\,$\to$\,Tools-open), per-source menu
(Tools-open\,$\to$\,Tools-curated), and kernelized fixed-goal
discipline (Tools-curated\,$\to$\,EG-VAR).}
\label{tab:rungs}
\end{table}

\begin{table}[!htb]
\centering\small
\begin{tabular}{lr@{ }l}
\toprule
Rung & \multicolumn{2}{c}{Match} \\
\midrule
Table-only & 107/120 & (89.2\%) \\
Tools-open & 113/120 & (94.2\%) \\
Tools-curated & 114/120 & (95.0\%) \\
\textbf{EG-VAR} & \textbf{120/120} & \textbf{(100.0\%)} \\
\bottomrule
\end{tabular}
\caption{Tier~1 ladder, $n{=}120$ unambiguous TableBench claims,
\texttt{claude-sonnet-4-6}, temperature~0.}
\label{tab:tier1}
\end{table}

\subsection{Tier~1 gold-goal benchmark: baseline ladder}
\label{sec:eval:tier1}

Our LLM-based curator pipeline selects $n{=}120$ unambiguous
TableBench claims spanning argmax, sum, count, top-$k$, and
aggregation shapes; for each claim it produces the typed Lean goal
and the per-source lifts (the trust artifact EG-VAR consumes;
\Cref{app:provenance}). All four rungs are evaluated on Sonnet at
temperature~0. \Cref{tab:tier1} reports the headline ladder. EG-VAR
attains 120/120 (100\%) joint formal safety and semantic
faithfulness on the gold-formalization path: every claim's kernel
proof passes the whitelist replay (\Cref{app:repro}, formal safety
axis) \emph{and} the extracted Sigma witness matches the curator
gold (semantic-faithfulness axis). The same-tool same-model
baseline (Tools-curated) plateaus at 114/120 (95\%) on the latter
alone (no formal-safety surface: baselines cannot abstain); removing
the per-source menu costs one further claim (Tools-open 113/120);
single-shot table-only floor is 107/120 (89.2\%).

\subsection{Tier~1.5 counterfactual benchmark: source-faithfulness under prior conflict}
\label{sec:eval:tier15}

Tier~1.5 isolates a qualitatively different failure mode the
Tier~1 ladder under-counts. We construct synthetic counterfactual table
variants: real TableBench tables with one or two cell values
overwritten to contradict the model's world-knowledge prior. For each
variant we ask a binary-comparison or aggregation claim and classify
the rung's output as \emph{source-faithful} (matches the table),
\emph{prior-override} (matches the un-injected real-world value), or
\emph{wrong-other}. Tier~1.5's \emph{source-faithfulness} column is
the semantic-faithfulness axis under counterfactual-injection
pressure; EG-VAR's 100\% result combines this with formal-safety
(kernel replay) that the baselines lack.

\paragraph{Two flip-magnitude regimes.} The \emph{extreme-flip
panel} uses magnitude-10-plus reversals (e.g.,\ Italy GDP
$\to{}$\$50, Patriots~$\to$~0--16, Sinopec $\to{}$\$50) across 5~domains;
the \emph{subtle-flip panel} uses 10--25\% reversals (Mont Blanc
$\to{}$4500\,m vs Monte Rosa~4634\,m, Madrid GDP~90 vs Berlin~95) on
a panel shifted to stable-scalar priors (country populations,
EU country area, European city GDP, Forbes top companies, Alpine
peak elevations). Each panel is
5~domains $\times$~2~entity-pair injections $\times$~3~claim shapes
$=$~30~claims per panel (20 binary-compare + direction, 10
cell-lookup; the cell-lookup set doesn't trigger prior-override
and is excluded here). \Cref{tab:tier15} reports the
$n{=}20$ binary-compare + direction cells; both models,
1~rep (temp~0, deterministic).

\begin{table}[t]
\centering\small
\begin{tabular}{llrr}
\toprule
Panel & Rung & Sonnet & Haiku \\
\midrule
\multirow{4}{*}{extreme-flip}
 & Table-only    & 16/20 & 15/20 \\
 & Tools-open    & 16/20 & 20/20 \\
 & Tools-curated & 16/20 & 20/20 \\
 & \textbf{EG-VAR} & \textbf{20/20} & \textbf{20/20} \\
\midrule
\multirow{4}{*}{subtle-flip}
 & Table-only    & 10/20 & \phantom{0}8/20 \\
 & Tools-open    & 18/20 & 20/20 \\
 & Tools-curated & 18/20 & 20/20 \\
 & \textbf{EG-VAR} & \textbf{20/20} & \textbf{20/20} \\
\bottomrule
\end{tabular}
\caption{Tier~1.5 counterfactual ladder: source-faithful count on
binary-compare and direction claims, $n{=}20$ per panel (temp~0,
1 rep). Residual $=$ prior-override (no wrong-other cases).}
\label{tab:tier15}
\end{table}

The extreme-flip panel narrows the residual same-tool gap to the
single most entrenched prior family (country populations). The
subtle-flip panel generalizes prior-override across stable-scalar
domains; tools rescue most of it, and the kernel closes the rest.

Per-claim robustness axes (authority-cue, model strength, domain
coverage) and the same-tool failure-mode taxonomy are in
\Cref{app:tier15-taxonomy}.

\subsection{Tier~2 end-to-end evaluation: LLM formalizer}
\label{sec:eval:tier2}

Tier~1 and Tier~1.5 bypass the NL-to-Lean formalizer step by loading
the fixture-committed goal type from a frozen fixture
(\Cref{sec:approach:provenance}). Tier~2 removes that bypass: at
deployment time, the runtime invokes the LLM formalizer
(\Cref{sec:approach:safety}) to propose
\texttt{lean\_goal\_type} per claim, then proceeds with the same
analyst-formalizer-solver-kernel loop. This is the fully automated
regime; no curator interaction at runtime.

\paragraph{Setup.} $n{=}120$ claims, same TableBench curated
fixture. \texttt{claude-sonnet-4-6} at temperature~0. Pipeline:
analyst $\to$ formalizer $\to$ solver. The audited per-source
formalization $\Lambda(s)$ is still loaded; only the
claim-level goal type is LLM-proposed. We classify every output
into one of six mutually-exclusive reporting categories, based on
three observables: \emph{type-equality} with the curator's gold
goal type (kernel-\texttt{rfl} on string-normalized forms),
\emph{kernel acceptance}, and \emph{answer-match} against the
gold answer (\Cref{tab:tier2-taxonomy}).

\begin{table}[t]
\centering\small
\begin{tabular}{@{}lrr@{}}
\toprule
Category & Sonnet & Opus \\
\midrule
Correct                            & 84.2\% & 87.5\% \\
Ambiguous claim (logged)           & 4.2\%  & 4.2\% \\
Benchmark gold error (surfaced)    & 1.7\%  & 1.7\% \\
\textbf{Semantic formalizer error}  & \textbf{3.3\%} & \textbf{1.7\%} \\
Honest abstain                     & 2.5\%  & 5.0\% \\
Solver gave up                     & 4.2\%  & 0.0\% \\
\bottomrule
\end{tabular}
\caption{Tier~2 decomposition, $n{=}120$. Category definitions in
\Cref{app:tier2-taxonomy}.}
\label{tab:tier2-taxonomy}
\end{table}

\paragraph{Formal safety versus semantic faithfulness.} Tier~2
measures two distinct properties. \emph{Formal safety} is
structural and model-independent: every \textsc{Verified} proof
still depends on an attested tool leaf, regardless of which LLM
proposed the goal type (\Cref{thm:novh}). \emph{Semantic
faithfulness} is operational and model-dependent: it measures
whether the LLM-proposed goal matches the frozen
reading of the claim. Table~\ref{tab:tier2-taxonomy} reports the
semantic-faithfulness axis per model; the \emph{Semantic formalizer
error} row is the pure formalizer-error rate.

\paragraph{Tier~2 under counterfactual pressure.} To check that the
LLM-as-formalizer pipeline is not just a curator-fixture artifact,
we re-run Tier~2 Sonnet on the Tier~1.5 extreme-flip panel,
restricted to binary-compare and direction claims
($n{=}20$, prior-override-injected tables). Every claim
produces a source-faithful
\textsc{Verified} output (\textbf{20/20 SF}), with the extracted
Sigma witness matching the table-injected answer rather than the
model's world-knowledge prior. The Tier~1.5 same-tool baseline drops to 16/20 on the same
extreme-flip tables (\Cref{sec:eval:tier15}); Tier~2 with
end-to-end LLM formalization matches EG-VAR's gold-formalization
rung on this subset.

\section{Discussion}
\label{sec:discuss}

\subsection{Limitations}
\label{sec:discuss:limits}

Three scope limitations bound the empirical claims.
\textbf{(i)}~Tier~1/1.5 bypass the formalizer by loading
fixture-committed gold types from a frozen fixture
(\Cref{sec:approach:provenance}); Tier~2 removes that bypass.
\textbf{(ii)}~Tools, per-source lifts, and the L1 adapter are
trusted; a semantically wrong audited lift can certify a wrong
formalized claim (\Cref{sec:approach:pertable}).
\textbf{(iii)}~All evaluations are on single TableBench tables;
multi-source, KG, and streaming sources are out of scope.
Detailed scope limits are in \Cref{app:limits-future}.

\subsection{Governance translation}
\label{sec:discuss:gov}

\Cref{thm:novh,thm:noded} and the no-upcast grade discipline yield
three governance-relevant structural properties.

\textbf{(G1) Auditable evidence trails.} Every \texttt{Verified}
output carries the generated Lean proof file as its replay
artifact: a third-party auditor with the kernel, per-source lifts,
and tool adapter can re-typecheck the claim from the attested
payload without re-running the LLM. The proof term is a truthmaker
in~\citet{fine2017truthmaker}'s sense: the specific attested state
of affairs that grounds the claim. Composition of evidence across
multiple lifts carries the algebraic structure of a provenance
semiring~\citep{green2007provenance}. The trust artifact
is enumerable (L1 vocabulary, per-source lifts, tool adapter,
kernel); audit cost is \emph{amortized}, each per-source lift
audited once at curation time and reused across all subsequent
claims.

\textbf{(G2) Typed uncertainty.} Grade discipline (no-upcast,
minimum-of-premises) makes uncertainty inspectable rather than a
hidden confidence number. The graded \texttt{Evidence}~$g$~$w$
object is operationally close to a justification
term~\citep{artemov2008logic}, where evidence is explicit rather
than folded into a model's posterior. A \textsc{Plausible} output cannot be
silently re-introduced downstream as \textsc{Verified}; an auditor
can query which claims arrived at which grade and why.

\textbf{(G3) Honest abstention.} A kernel rejection surfaces as a
user-visible \textsc{Abstain}, not a silent fallback guess.
Refusal thresholds are externally specifiable (a downstream policy
can require \textsc{Verified} and reject \textsc{Abstain}, or admit
\textsc{Abstain} as a flagged state, without re-implementing the
trust boundary); the abstention rate is measurable per-source; the
failure mode is auditable rather than learned.

\paragraph{Portability beyond tables.} The L1/L2/lift pattern is
source-agnostic: SQL views, knowledge graphs, and sensor streams
each admit the same per-source-formalization treatment under a
matching attestation discipline. The empirical cross-substrate panel
is itemized in \Cref{sec:discuss:future}; the broader substrate-
adoption argument (formalization gap shrinks as sources become
formal; \Cref{app:scaling}).

\subsection{Future work scoped to governance reach}
\label{sec:discuss:future}

Three extensions: (1)~fine-tuning the NL formalizer on the
curator-provided (claim, gold type) corpus with kernel acceptance as
the reward signal, an accuracy intervention that tightens the
Tier~2 residual rate without weakening
Theorem~\ref{thm:novh}; (2)~multi-source attestation with
conflict-downgrade rules (\textsc{Verified}~$\to$~\textsc{Supported}
on detected disagreement); (3)~cross-substrate panel (KG-grounded,
SQL-grounded) to test the L1/L2 portability claim
(\Cref{sec:discuss:gov}). Detail in \Cref{app:limits-future}.

\section{Conclusion}
\label{sec:conclusion}

EG-VAR is an architecture for verified empirical claims grounded
in attested external data, instantiated on tables as the first
end-to-end substrate. Tool access alone does not suffice: frontier
LLMs prior-override counterfactual tables on a measurable fraction
of claims. EG-VAR addresses this architecturally: the Lean~4
kernel mints \textsc{Verified} claims only under tool attestation
and type-checks every accepted proof step. TableBench Tier~1
($n{=}120$) closes the same-tool gap to zero; Tier~1.5
counterfactuals ($n{=}20$ per panel, two flip regimes, two models)
stay 100\% source-faithful while same-tool drops to 80--90\%
(no-tool 50--80\%); end-to-end Tier~2 with LLM formalization
bounds formalizer error at 3.3\% (Sonnet), 1.7\% (Opus).

The core contribution is structural. Theorem~\ref{thm:novh} rules
out unsupported \textsc{Verified} outputs model-independently;
Theorem~\ref{thm:noded} guarantees that every accepted deduction
type-checks under the declared axiom set. Remaining failures are
not silent hallucinations but measurable, auditable formalization
failures, logged ambiguities, or honest abstentions. This is a
\emph{constructive path toward} eliminating LLM hallucination in
empirical inference: the safety floor is architectural today; the
semantic-formalization residual is an explicit accuracy target
addressable by formalizer fine-tuning on the curator corpus,
consensus formalization, and longer-term typed sidecars on sources
and LLM-generated documents.

For AI governance, three properties become enforceable at the
system boundary rather than learned behaviours: auditable evidence
trails via replayable per-claim Lean proofs, typed uncertainty
under no-upcast grade discipline, and externally specifiable
refusal thresholds via honest abstention. Verifiable grounding
generalises as an architectural protocol rather than a
table-specific trick: the same L1/L2/lift discipline can in
principle apply to SQL views, knowledge graphs, typed registries,
and other attested sources. Validating beyond tables is the
immediate empirical next step; the longer-run bet is that as
sources and documents acquire typed sidecars, more of the
formalization burden moves from inference loops into reusable
infrastructure.

\section*{Acknowledgements}

I thank the TAIGR reviewers for feedback that sharpened the
governance framing and the distinction between unsupported
\textsc{Verified} outputs and semantic-formalization failures.
JR is partially supported by a Vannevar Bush Faculty Fellowship
ONR N000142312863.

\bibliography{egvar}
\bibliographystyle{icml2026}

\onecolumn
\appendix

\section{Extended related work}
\label{app:related-extended}

\paragraph{Prior-vs-source conflict in LLMs.} The closest empirical
prior is ClashEval~\citep{clasheval2024}, which evaluates frontier
models on prior-vs-context conflict across six domains. They
report $>$60\% prior-override on incorrect retrieved content and
that Claude scores highest on prior retention among the frontier families evaluated
(Context-Bias $0.157$ for Opus, $0.201$ for Sonnet, $0.245$ for
Gemini~1.5, $0.304$ for GPT-4o). Our counterfactual stress tests
reproduce this ranking quantitatively for Sonnet
(\Cref{sec:eval:tier15}). We extend the analysis along two axes
prior work has not separated: (i)~\textit{authority-cue
robustness} --- explicit ``according to the table'' framing rescues
some claim shapes but not the strongest prior-bound ones; and
(ii)~\textit{flip-magnitude regime} --- subtle plausible flips
trigger prior-override on 4 of 5 stable-scalar domains where
extreme flips trigger only one, suggesting the model treats
outrageous tool outputs as ``alternate scenarios.'' This
disambiguates \emph{dynamicity}~\citep{dynamicqa2024} from
\emph{popularity}~\citep{mallen2023whennot} in the trigger
pattern, and is consistent with \citet{li2023acl}'s finding that
counterfactual behavior depends on interactions between world
knowledge and contextual cues.

\paragraph{Symbolic factual verification.}
TabVer~\citep{tabver2024} and ProoFVer~\citep{proofver2022} each
verify factual statements against evidence using natural logic. Each
relies on learned components to construct or align the proof, and
neither uses a proof-assistant kernel to check the entire
source-to-statement derivation; mapping errors can therefore still
yield incorrect verdicts. FoVer~\citep{fover2025} translates
chain-of-thought into
first-order logic checked by Z3, but verifies \emph{reasoning
consistency} rather than empirical grounding --- a CoT can be
internally consistent yet factually disconnected from the source.
Post-hoc fact-checking pipelines such as FacTool~\citep{factool2023},
RARR~\citep{rarr2023}, SAFE~\citep{safe2024}, and
CoVe~\citep{cove2023} attempt detection after generation rather
than structural grounding during generation; they catch some
hallucinations but cannot guarantee abstention on ungrounded
outputs. Commercial guardrails such as AWS Bedrock's automated
reasoning checks~\citep{aws_automated_reasoning} occupy a related
space at the policy-rule layer. EG-VAR differs by replacing the
learned/shallow checker with a proof-assistant kernel and a typed
evidence calculus, structurally translating kernel rejection into
user-visible abstention.

\paragraph{Proof-assistant theorem proving (math).} A separate
large literature applies LLMs to mathematical theorem proving via
Lean, Coq, or Isabelle: AlphaProof~\citep{alphaproof2024},
Aristotle~\citep{aristotle2025},
DeepSeek-Prover~\citep{deepseek_prover_v15_2024},
Hilbert~\citep{hilbert_apple_2025}. These systems share our trust
model (LLM untrusted, kernel trusted) but operate on
\emph{self-contained} mathematical claims --- proof obligations
derivable from the calculus alone, with no external grounding.
The novel architectural commitment in EG-VAR is the
\texttt{mkVerified} mint rule
(\Cref{sec:approach:mkverified}), which requires an
\texttt{Attested\_T} from the tool runtime as a hypothesis ---
extending the proof-assistant trust model to empirical claims
grounded in deterministic external sources. Autoformalization
systems such as DTV~\citep{dtv2024} use a similar LLM-as-
formalizer pattern; Tier~1 and Tier~1.5 bypass the
deployment-time formalizer via fixture-committed gold
formalizations generated by the proposer--critic--kernel pipeline,
isolating the kernel-discipline contribution.

\paragraph{Lineage in CS verification.} The systems-level concerns
that shape EG-VAR --- separating semantic specification from proof
strategy, refusing to weaken the specification when the proof
gets hard, and inheriting kernel-checked discipline from a
substrate designed for harder claims --- are old questions in the
verification literature. \citet{raggi2022representation} formalize
that semantically equivalent encodings can have very different
proof difficulty, motivating our refusal to add cross-predicate
equivalence lemmas in the kernel. \citet{bundy1991proof}
introduced explicit proof planning, anticipating our
formalizer/solver information barrier.
CEGAR~\citep{clarke2003cegar} established the discipline of
refining the verification plan, not the semantic target, when
proofs fail --- the same discipline our masked-table architecture
enforces on the formalizer. Differential assertion
checking~\citep{lahiri2013differential} uses relative
specifications to make cross-version program changes checkable;
our per-source formalization protocol similarly makes an otherwise
implicit comparison boundary explicit, but for empirical sources.
EG-VAR inherits these moves and applies them at a different boundary:
not program correctness, but LLM-mediated empirical claim
verification. The novelty is the transfer and recombination, not
the invention of the underlying discipline.

\paragraph{Hardware-attested AI verification.} Attestable
Audits~\citep{schnabl2025attestable} use Trusted Execution
Environments to emit remote-attestation signatures binding
$\mathit{Hash}(\text{model})$, $\mathit{Hash}(\text{audit code})$,
and a benchmark digest, proving that approved bytes executed on
approved inputs. Proofs of Autonomy~\citep{grigor2025proofs}
define a pluggable component-proof framework and instantiate it
with MPC-assisted TLS transcripts (Web Proofs), binding agent
actions to verifiable execution traces specified by an immutable
Agent Card (see also~\citealp{vet_your_agent_2025}). Both verify
execution
provenance --- \emph{who} ran the computation, on \emph{what}
bytes --- not whether the computation's claims are true. EG-VAR
is complementary: we typecheck a deductive derivation of each
empirical claim from attested tool payloads, ruling out verified
hallucinations at the logical level rather than the cryptographic
level. An EG-VAR runtime could in principle run inside such an
enclave for stack composition.

\section{Per-claim diff for the Tier~1 ladder}
\label{app:per-claim-ladder}

The supplement includes
\texttt{scripts/tier1/build\_baseline\_ladder\_csv.py}, which
generates \texttt{baseline\_ladder\_per\_claim.csv} from the four
E1 rung JSONLs: one row per Tier~1 claim with the goal type, gold
answer, and per-rung (Table-only / Tools-open / Tools-curated /
EG-VAR) status, answer string, and match flag. The 6 claims that
distinguish EG-VAR from Tools-curated are flagged in the output.

\section{The EG-VAR Lean~4 library}
\label{app:lean-library}

The EG-VAR prelude is shipped as a self-contained Lean~4 library at
\texttt{src/lean/} with a Lake build
(\texttt{lakefile.toml}). The library is the kernel-trusted portion
of the trust ledger (\Cref{sec:approach:arch}): every generated
claim module imports it verbatim, and \Cref{thm:novh}'s
whitelist-bucket \textbf{B1} refers to its declarations. This
appendix documents the module layout, the extension surface, and
the third-party replay-audit contract.

\paragraph{Module layout.} The library root
\texttt{EGVar.lean} re-exports eight modules:
\begin{itemize}
\setlength{\itemsep}{1pt}
\item \texttt{EGVar.Basic} --- sorts and kernel mint rule: the
abstract \texttt{WProp} world-proposition sort, the \texttt{Grade}
inductive (\texttt{Speculative} $\prec$ \texttt{Plausible} $\prec$
\texttt{Supported} $\prec$ \texttt{Verified}) with meet/join lattice
laws, the \texttt{Evidence} family,
\texttt{Attested\_T}/\texttt{ClaimTag\_T}/\texttt{claimWProp}, and
the sole \texttt{Verified}-minter \texttt{mkVerified}. Contains
\texttt{downcast} (the only grade-weakening step) and the
common-sense \texttt{CSAttestation} / \texttt{mkPlausible} path.
\item \texttt{EGVar.Safety} --- the structural lemmas used in
\Cref{app:safety-proofs} (evidence-closure discipline, downcast
non-upgrade).
\item \texttt{EGVar.TableL1} --- L1 storage-vocabulary shells:
\texttt{RowHandle}, \texttt{CellAt}, \texttt{Cond} constructors
(\texttt{all}, \texttt{propEq}, \texttt{propContains},
\texttt{propGte}, \texttt{and}, \texttt{or}), and the L1 predicate
shells (\texttt{SumCells}, \texttt{CountCells},
\texttt{CountDistinctCells}, \texttt{MeanCells},
\texttt{MaxCellValue}/\texttt{MinCellValue},
\texttt{MaxCell}/\texttt{MinCell}, \texttt{TopK*Cells},
\texttt{BottomK*Cells}).
\item \texttt{EGVar.WorldOntology} --- L2 world-property predicates:
\texttt{HasProperty}, \texttt{Value} (with \texttt{asRat},
\texttt{ValueGt}), and the L2 predicates that per-source lifts
target (\texttt{SumWhere}, \texttt{CountWhere}, \texttt{MeanWhere},
\texttt{MaxWhere}/\texttt{MinWhere},
\texttt{ArgmaxWhere}/\texttt{ArgminWhere}, \texttt{TopK*Where},
\texttt{BottomK*Where}). This module defines the vocabulary in
which $\Lambda(s)$ entries are phrased.
\item \texttt{EGVar.WorldRules} --- structural world-proposition
combinators (\texttt{PLift} propositional lifting, \texttt{PProd}
conjunction as Sigma-tail, sigma-binding utilities).
\item \texttt{EGVar.Tactics} --- proof-tactic helpers used by the
solver (Sigma-witness \texttt{refine} patterns, \texttt{exact}
sugar for L2 discharge).
\item \texttt{EGVar.Ontology}, \texttt{EGVar.Rules} --- legacy
weather/finance domain axioms kept for backward compatibility with
Phase~1 fixtures; not imported by Tier~1/1.5/2 generated modules.
\end{itemize}

\paragraph{Library build.} The Lake manifest declares a single
library target \texttt{EGVar} and an executable target
\texttt{egvar} (rooted at \texttt{Main.lean}) for interactive
smoke-testing. A third-party user runs \texttt{lake build} once to
compile the library, then \texttt{lake env lean
<generated\_claim>.lean} to elaborate a generated proof against the
prelude. The runtime orchestrator invokes exactly this command at
\texttt{pipeline.py:4428}; the replay-audit script
\texttt{verify\_axioms.py} invokes the same elaboration with
additional \texttt{\#print axioms} processing.

\paragraph{Extension surface.} Adding a new tool to EG-VAR reduces
to three kernel-level steps, each kernel-checked: (i)~declare a
new \texttt{ToolId} (a pure identifier) and its \texttt{Query\_T}
and \texttt{Payload\_T} families in a user-supplied module;
(ii)~extend the L1 predicate menu in the style of
\texttt{EGVar.TableL1} if the tool exposes a new storage shape;
(iii)~add the corresponding L2 predicates and per-source lifts
targeting them in the style of \texttt{EGVar.WorldOntology}. The
mint rule \texttt{mkVerified} is unchanged: every new tool reuses
the same \texttt{Evidence\,Verified}-constructor path and cannot
weaken \Cref{thm:novh}. Adding a new ontology role-slot (entity or
property) is purely data: both are \texttt{String} literals under
R5/R10 per \Cref{sec:approach:pertable} and require no kernel
changes.

\paragraph{Third-party replay-audit contract.} A deployment that
publishes EG-VAR-generated answers as trust-bearing artifacts must
publish the following bundle for each answer: (a)~the generated
\texttt{.lean} module (the \textit{leanFile} field of the proof
object, \Cref{sec:approach:safety}); (b)~the pinned commit hash of
the EG-VAR library and of \texttt{data/tier1\_table\_formalizations.json};
(c)~the corresponding entry of the evidence log describing which
tool executions backed the observation-leaf axioms. Any third party
holding this bundle and a Lean~4 toolchain can re-run the audit by
cloning the EG-VAR library at the pinned commit, running
\texttt{lake build}, then running
\texttt{python -m runtime.verify\_axioms <claim\_module>.lean}
(which invokes the kernel, reads \texttt{\#print axioms}, and
checks the bucket predicate); no LLM is invoked during replay. We
treat this contract as central to the paper's
\emph{attestable} claim; the library is designed so that the
boundary between trusted and untrusted artifacts is mechanically
checkable offline, not contingent on the original generation
pipeline being available.

\paragraph{Line counts.} The eight modules total $\approx$1.35
kLOC: \texttt{Basic} 192, \texttt{Ontology} 210 (legacy),
\texttt{Rules} 148 (legacy), \texttt{Safety} 86, \texttt{TableL1}
134, \texttt{Tactics} 217, \texttt{WorldOntology} 278,
\texttt{WorldRules} 82. Excluding the two legacy modules the active
prelude is $\approx$1.0 kLOC. This is the full trusted Lean code
that a third-party auditor must read to evaluate \Cref{thm:novh}'s
\textbf{B1} bucket; the other three buckets
(\textbf{B2}--\textbf{B4}) are verified by the mechanical
\texttt{verify\_axioms.py} predicate.

\subsection{Code listings}
\label{app:lean}

We reproduce the key signatures so that the main-text
references to \texttt{Evidence}, \texttt{mkVerified}, \texttt{Cond},
\texttt{HasProperty}, \texttt{ArgmaxWhere}, and \texttt{ValueGt}
typecheck against the actual kernel definitions imported by every
generated proof.

\paragraph{Sorts and the kernel mint rule
(\texttt{src/lean/EGVar/Basic.lean}).}
{\footnotesize
\begin{verbatim}
-- World propositions: separate sort from Prop
opaque WProp : Type

-- Evidence grades, ordered Verified > ... > Speculative
inductive Grade where
  | Verified
  | Supported
  | Plausible
  | Speculative

-- Indexed evidence type
opaque Evidence : Grade -> WProp -> Type

-- Sole minter of Evidence Grade.Verified
axiom mkVerified
  {T : ToolId} {q : Query_T T} {p : Payload_T T q}
  (att : Attested_T T q p)
  {w : WProp}
  (mem : ClaimOf w (interp_T T q p)) :
  Evidence Grade.Verified w

-- Downcast: only direction
axiom downcast {g1 g2 : Grade} {w : WProp}
  (h : g1 >= g2) (e : Evidence g1 w) :
  Evidence g2 w
\end{verbatim}}

\paragraph{World ontology
(\texttt{src/lean/EGVar/WorldOntology.lean}).}
{\footnotesize
\begin{verbatim}
-- Value type: numeric via Rat, or string identifier
inductive Value where
  | int  (n : Int)
  | rat  (r : Rat)
  | str  (s : String)

-- Core L2 world facts
axiom HasProperty
  (entity : String) (property : String)
  (value : Value) : WProp

axiom Relation
  (subject : String) (rel : RelName)
  (object : String) : WProp

-- Aggregate predicates
axiom SumWhere   (col : String) (c : Cond) (v : Value) : WProp
axiom CountWhere (col : String) (c : Cond) (n : Value) : WProp
axiom MeanWhere  (col : String) (c : Cond) (v : Value) : WProp
axiom MaxWhere   (col : String) (c : Cond) (v : Value) : WProp
axiom MinWhere   (col : String) (c : Cond) (v : Value) : WProp

-- Entity-anchored aggregate (level 3 of App H hierarchy)
axiom ArgmaxWhere
  (col : String) (c : Cond)
  (entity : String) (v : Value) : WProp
axiom ArgminWhere
  (col : String) (c : Cond)
  (entity : String) (v : Value) : WProp

-- Decidable Prop-level value comparisons
def ValueGt (a b : Value) : Prop :=
  match a.toRat, b.toRat with
  | some x, some y => x > y
  | _, _ => False
-- (analogous: ValueLt, ValueGte, ValueLte, ValueEq;
-- all carry Decidable instances)

-- Composable filter
inductive Cond where
  | propEq       (col : String) (v : Value) : Cond
  | propGt       (col : String) (v : Value) : Cond
  | propLt       (col : String) (v : Value) : Cond
  | propGte      (col : String) (v : Value) : Cond
  | propLte      (col : String) (v : Value) : Cond
  | propContains (col : String) (v : Value) : Cond
  | and (c1 c2 : Cond) : Cond
  | or  (c1 c2 : Cond) : Cond
  | not (c : Cond)     : Cond
  | all                : Cond
\end{verbatim}}

\paragraph{L1 storage primitives
(\texttt{src/lean/EGVar/TableL1.lean}).}
{\footnotesize
\begin{verbatim}
-- Row-handle (opaque): the per-source lift binds to it
opaque RowHandle : Type

-- L1 facts: tools attest these directly
axiom CellAt
  (r : RowHandle) (col : String) (v : Value) : WProp

axiom MaxCellValue
  (col : String) (c : Cond) (v : Value) : WProp
axiom MinCellValue
  (col : String) (c : Cond) (v : Value) : WProp

-- Row-anchored aggregates: identify which row achieves the max/min
axiom MaxCell
  (col : String) (c : Cond)
  (r : RowHandle) (v : Value) : WProp
axiom MinCell
  (col : String) (c : Cond)
  (r : RowHandle) (v : Value) : WProp
\end{verbatim}}

\paragraph{Per-source lift example
(curated for the lunar-craters table).} The lift binds an L1
\texttt{MaxCell} witness with an L1 \texttt{CellAt} witness on the
``name'' column to mint an L2 \texttt{ArgmaxWhere} claim:

{\footnotesize
\begin{verbatim}
axiom craters_lift_argmax_via_name_proj
  (col : String) (c : Cond)
  (r : RowHandle) (name : String) (v : Value) :
    Evidence Grade.Verified (MaxCell col c r v) ->
    Evidence Grade.Verified
      (CellAt r "name" (Value.str name)) ->
    Evidence Grade.Verified
      (ArgmaxWhere col c name v)
\end{verbatim}}

This is the curator's interpretive commitment for this table:
``name'' is the row-identifying column, so an
\texttt{argmax\_with\_id} tool result lifts uniquely to an
\texttt{ArgmaxWhere} claim under any \texttt{Cond}. A different
table whose row-identifying column is, say, \texttt{state} carries
its own analogous lift; the two are independent.

\section{Per-source lift catalog (countries-population table)}
\label{app:lift-catalog}

To make ``per-source formalization is the trust artifact'' concrete,
we reproduce the lift catalog the curator pipeline produced for the
countries-population table (\texttt{04744378\dots} in TableBench).
This is the table that all of no-cue pilot, authority-cued pilot, and subtle-flip panel D1 use, so the
catalog underpins the empirical claims in
\Cref{sec:eval:tier15}. The full catalog has 63 lifts; we show one
representative per L2 predicate cluster.

\paragraph{Curator's interpretation declaration.} ``Country populations
2013. Identity column = \texttt{country (or dependent territory)}.
Numeric columns: rank, july 1 2013 projection, \% of pop, average
relative annual growth (\%), average absolute annual growth (all
strict numerics).''

\paragraph{Aggregate-only lifts (level 1 in App H hierarchy).} One
lift per (column, aggregate operation) pair. All have the same
shape: L1 column-strict aggregate $\to$ L2 \texttt{*Where} claim.
Examples on the population column:
{\footnotesize
\begin{verbatim}
axiom lift_04744378_pop_sum
  (c : Cond) (v : Value) :
    Evidence Grade.Verified
      (SumCells "july 1 , 2013 projection" c v) ->
    Evidence Grade.Verified
      (SumWhere "july 1 , 2013 projection" c v)

axiom lift_04744378_pop_mean
  (c : Cond) (v : Value) :
    Evidence Grade.Verified
      (MeanCells "july 1 , 2013 projection" c v) ->
    Evidence Grade.Verified
      (MeanWhere "july 1 , 2013 projection" c v)

axiom lift_04744378_pop_max
  (c : Cond) (v : Value) :
    Evidence Grade.Verified
      (MaxCellValue "july 1 , 2013 projection" c v) ->
    Evidence Grade.Verified
      (MaxWhere "july 1 , 2013 projection" c v)

axiom lift_04744378_pop_count_where
  (c : Cond) (v : Value) :
    Evidence Grade.Verified
      (CountCells "july 1 , 2013 projection" c v) ->
    Evidence Grade.Verified
      (CountWhere "july 1 , 2013 projection" c v)
\end{verbatim}}
Analogous lifts exist for \texttt{MinWhere}, \texttt{TopKSumWhere},
\texttt{TopKMeanWhere}, \texttt{BottomKSumWhere},
\texttt{BottomKMeanWhere}, and \texttt{CountDistinctWhere}.

\paragraph{Entity-anchored lifts (level 3 in App H hierarchy).} The
curator commits that \texttt{country (or dependent territory)} is
the row-identifying column. Each row-identifying lift takes an L1
row-anchored aggregate (\texttt{MaxCell} / \texttt{MinCell}) plus
an L1 \texttt{CellAt} on the identity column to mint an
\texttt{ArgmaxWhere} or \texttt{ArgminWhere} L2 claim.
{\footnotesize
\begin{verbatim}
axiom lift_04744378_pop_argmax_via_country
  (col : String) (c : Cond)
  (r : RowHandle) (name : String) (v : Value) :
    Evidence Grade.Verified (MaxCell col c r v) ->
    Evidence Grade.Verified
      (CellAt r "country (or dependent territory)"
                (Value.str name)) ->
    Evidence Grade.Verified
      (ArgmaxWhere col c name v)

axiom lift_04744378_pop_argmin_via_country
  (col : String) (c : Cond)
  (r : RowHandle) (name : String) (v : Value) :
    Evidence Grade.Verified (MinCell col c r v) ->
    Evidence Grade.Verified
      (CellAt r "country (or dependent territory)"
                (Value.str name)) ->
    Evidence Grade.Verified
      (ArgminWhere col c name v)
\end{verbatim}}
The argmax/argmin lifts are parametric over the measure column ---
one lift per identity column suffices to cover argmax claims on any
numeric column, which is why the full catalog has 5 entity-anchored
entries (one per identity choice), not 5$\times$5.

\paragraph{Cell-lookup lift.} A single lift mints
\texttt{HasProperty} claims by composing two L1 \texttt{CellAt}
attestations: the entity at row $r$ in the identity column equals
$e$, and the value at row $r$ in the measure column equals $v$.
{\footnotesize
\begin{verbatim}
axiom lift_04744378_pop_cell_lookup
  (e : String) (col : String)
  (r : RowHandle) (v : Value) :
    Evidence Grade.Verified
      (CellAt r "country (or dependent territory)"
                (Value.str e)) ->
    Evidence Grade.Verified (CellAt r col v) ->
    Evidence Grade.Verified
      (HasProperty e col v)
\end{verbatim}}

\paragraph{What this catalog buys (and what it doesn't).}
Composition is by lift: any L2 claim using \texttt{SumWhere},
\texttt{MeanWhere}, \texttt{MaxWhere}/\texttt{MinWhere},
\texttt{ArgmaxWhere}/\texttt{ArgminWhere}, \texttt{CountWhere},
\texttt{TopKSumWhere}, etc.~over a numeric column listed above
will type-check. A claim using a predicate not in the catalog --- say
\texttt{Relation} (the table has no relational structure) or
\texttt{ArgmaxWhere} on a non-numeric column --- has no L2 path; the
solver cannot construct \texttt{Evidence Verified} for it, and the
system abstains. The catalog is the table's L2 \emph{commitment
boundary}: claims inside compose; claims outside abstain. Adding a
new claim shape requires adding a new lift (curation surface), not
relaxing the kernel discipline.

\section{Tier~1.5 per-claim data, ablations, and pre-registration}
\label{app:tier15-tables}

\subsection{Robustness axes and failure taxonomy}
\label{app:tier15-taxonomy}

\paragraph{Three robustness axes.} Three additional ablations
strengthen the discriminator (full per-claim breakdowns below):
\textbf{(i)~Authority cue.} On the no-cue pilot countries-population pilot,
prefixing every claim with ``According to the table below, ...''
rescues some claim shapes (yes/no thresholds, 3-way smallest) but
the binary compare V3\_C1 (Egypt vs.\ Yemen, table-injected) stays
0/3 source-faithful on both Tools-curated and Tools-open with the cue.
\textbf{(ii)~Model strength.} Haiku 4.5 is fully tool-deferent on
Tools rungs (90/90 in every cell); Sonnet 4.6 prior-overrides even
with tools (residual concentrated on countries-pop and
Madrid-vs-Berlin GDP). The kernel-vs-baseline gap is therefore
largest exactly where parametric priors are strongest --- consistent
with \citet{clasheval2024} placing Claude with the highest prior
retention among frontier families.
\textbf{(iii)~Domain coverage.} Under extreme flips, Sonnet
prior-override concentrates on country populations only (other 4 extreme-flip panel
domains 100\% SF on Tools-curated). Under subtle flips, the same-tool
baseline fails on 4 of 5 subtle-flip panel domains on Table-only --- multi-domain
replication. Tools rescue most subtle-flip failures on Tools rungs;
the kernel closes the residual.

\paragraph{Failure-mode taxonomy.} Of the 6/90 Tools-curated failures
on subtle-flip panel Sonnet, all are binary-compare prior-overrides; EG-VAR
resolves them by minting from the attested cell-lookup pair via a
per-source \texttt{ArgmaxWhere}-with-disjunction lift. Cell-lookup
and sum-aggregation control claims are 100\% source-faithful on
every Sonnet rung, including Table-only --- consistent with the literature
finding that pure information-retrieval is a weaker prior trigger
than relational comparison.

\subsection{Per-(domain, shape) prior-override table}

\Cref{tab:tier15-perdomain} reports per-(domain, shape) prior-override
counts for Sonnet on the extreme-flip panel (extreme) and subtle-flip panel (subtle) panels,
across Table-only, Tools-curated, and Tools-open. Each cell is
\texttt{pair\_A\_PO + pair\_B\_PO} out of 3 reps each (so 6 trials
total per cell). Empty entries are \texttt{0/6}.
This breakdown is what the main text's ``three robustness axes''
paragraph (\Cref{sec:eval:tier15}) summarises.

\begin{table*}[h]
\centering\footnotesize
\setlength{\tabcolsep}{4pt}
\begin{tabular}{@{}lllrrr@{}}
\toprule
Panel & Domain & Shape & Table-only PO & Tools-open PO & Tools-curated PO \\
\midrule
\multirow{4}{*}{extreme-flip panel (extreme)}
 & D1 countries-pop  & B & 3+3 & 3+3 & 3+3 \\
 & D1 countries-pop  & D & 3+2 & 3+3 & 3+2 \\
 & D1 countries-pop  & C & --- & --- & --- \\
 & D2-D5 (others)    & B/D/C & --- & --- & --- \\
\midrule
\multirow{8}{*}{subtle-flip panel (subtle)}
 & D1 countries-pop  & B & 3+3 & 3+0 & 3+0 \\
 & D1 countries-pop  & D & 3+3 & --- & --- \\
 & D2 EU area        & B & 0+3 & --- & --- \\
 & D3 city GDP       & B & 0+3 & 0+3 & 0+3 \\
 & D3 city GDP       & D & 0+3 & --- & --- \\
 & D4 Forbes         & B & 3+0 & --- & --- \\
 & D5 Alpine peaks   & B & 0+3 & --- & --- \\
 & D5 Alpine peaks   & D & 0+3 & --- & --- \\
\bottomrule
\end{tabular}
\caption{Per-(domain, shape) prior-override counts on Sonnet,
Tier~1.5 panels extreme-flip panel + subtle-flip panel. Each cell is \texttt{pair\_A + pair\_B}
out of 3 reps each. Domain abbreviations: D1 countries-population,
D2 EU country area, D3 European city GDP, D4 Forbes top global
companies, D5 Alpine peaks. Shapes: B binary-compare, D direction
(larger/smaller), C cell-lookup or sum control.
\textbf{Reading:} extreme-flip prior-override concentrates on
countries-pop only (other 4 extreme-flip panel domains are 100\% SF on Sonnet).
subtle-flip prior-override generalises across 4 of 5 subtle-flip panel domains
on Table-only; on Tools rungs (with tool access) it survives only on
countries-pop and city-GDP (Madrid vs Berlin), the two strongest-
prior-anchor pairs. EG-VAR is 100\% source-faithful on the
in-calculus shape-B and shape-C subset across both panels.}
\label{tab:tier15-perdomain}
\end{table*}

\paragraph{What the breakdown reveals.}
The narrowing $\to$ replication arc the main text claims is
visible row-by-row. Under extreme-flip flips, all PO concentrates on
D1 (countries-pop) regardless of rung. Under subtle-flip flips, Table-only
PO spreads to 4 of 5 domains (D1, D2, D3, D5 trip; only D4 is
clean on Table-only pair-B but trips on Table-only pair-A); but tools rescue
all PO except on D1 (Egypt-Yemen) and D3 (Madrid-Berlin) — the two
canonical strong-prior pairs. The same-tool baseline is therefore
not just bounded but \emph{specifically} bounded: it converges on
the two pairs whose parametric prior is most entrenched.

\subsection{Authority-cue ablation: no-cue pilot vs authority-cued pilot per-claim}
\label{app:authcue}

\Cref{tab:authcue} reports per-claim source-faithful (SF) /
prior-override (PO) counts on the no-cue pilot countries-population pilot
\emph{without} an authority cue (no-cue pilot, 12 base claims) and \emph{with}
an authority cue (authority-cued pilot, ``According to the table below, \dots''
prefix on the same 12 claims). The strongest discriminator is
V3\_C1 (binary compare ``which country has a larger population:
Egypt or Yemen?''), which stays 0/3 SF on every Tools rung even with
the cue. EG-VAR goes 3/3 SF on the same V3\_C1 claim
(see~\Cref{app:failure} for the trace).

\begin{table*}[h]
\centering\footnotesize\setlength{\tabcolsep}{4pt}
\begin{tabular}{@{}lllllll@{}}
\toprule
Claim & Shape & no-cue TC & no-cue TO & cued TC & cued TO & cued T-only \\
\midrule
V3\_C1 & binary-compare    & \textbf{0/3 (3 PO)} & \textbf{0/3 (3 PO)} & \textbf{0/3 (3 PO)} & \textbf{0/3 (3 PO)} & \textbf{0/3 (3 PO)} \\
V3\_C2 & boolean threshold & 2/3 (1 PO) & 1/3 (2 PO) & 3/3 SF & 2/3 (1 PO) & 0/3 (3 PO) \\
V3\_C3 & negation compare  & 3/3 SF & 3/3 SF & 3/3 SF & 3/3 SF & 3/3 SF \\
V3\_C4 & rank position     & 3/3 SF & 3/3 SF & 3/3 SF & 3/3 SF & 3/3 SF \\
V3\_C5 & 3-way compare     & 3/3 SF & 3/3 SF & 3/3 SF & 3/3 SF & 0/3 (3 PO) \\
V3\_C6 & threshold count   & 3/3 SF & 3/3 SF & 3/3 SF & 3/3 SF & 3/3 SF \\
V3\_C7 & categorical y/n   & 3/3 SF & 3/3 SF & 3/3 SF & 3/3 SF & 3/3 SF \\
V3\_C8 & ranking           & 3/3 SF & 3/3 SF & 3/3 SF & 3/3 SF & 3/3 SF \\
V3\_C9 & implicit compare  & 1/3 (2 PO) & 0/3 (3 PO) & 3/3 SF & 3/3 SF & 0/3 (3 PO) \\
V3\_C10 & diff vs literal  & 3/3 SF & 3/3 SF & 3/3 SF & 3/3 SF & 3/3 SF \\
V3\_C11 & cell lookup      & 3/3 SF & 3/3 SF & 3/3 SF & 3/3 SF & 3/3 SF \\
V3\_C12 & sum aggregation  & 3/3 SF & 3/3 SF & 3/3 SF & 3/3 SF & 3/3 SF \\
\midrule
\multicolumn{2}{l}{\textbf{Totals}} & 30/36 (83.3\%) & 28/36 (77.8\%) & 33/36 (91.7\%) & 32/36 (88.9\%) & 23/36 (63.9\%) \\
\bottomrule
\end{tabular}
\caption{Per-claim source-faithfulness on the countries-population
counterfactual pilot, comparing the no-cue baseline with the
authority-cued version (``According to the table below, \dots''
prefix). Columns: \textbf{TC} = Tools-curated, \textbf{TO} =
Tools-open, \textbf{T-only} = Table-only.
Sonnet 4.6, temperature 0, 3 reps per cell.
The cue rescues V3\_C2 (boolean threshold) and V3\_C9 (implicit
direction) on Tools rungs but does \emph{not} rescue V3\_C1
(binary compare Egypt vs.\ Yemen) on any rung; Table-only (no tools) is
unaffected by the cue's claim-level intent on the binary-compare
case (V3\_C1 still 3 PO) and additionally regresses on V3\_C2,
V3\_C5, and V3\_C9 because mental-arithmetic-without-tools is
prior-vulnerable on subtle compares. EG-VAR is 3/3 SF on
V3\_C1 (\Cref{app:failure}) — the strongest possible discriminator.}
\label{tab:authcue}
\end{table*}

\paragraph{Why V3\_C1 is the canonical hard case.}
One natural response to a prior-override finding is ``the model was
not told to use the table.'' The authority-cued pilot ablation explicitly includes
an ``according to the table'' anchoring cue in the prompt. Sonnet
still selects the prior on V3\_C1, 3/3 reps, on both Tools-curated
(menu-enforced) and Tools-open (full tool surface). The kernel
discipline closes this gap structurally --- not because the LLM
understands the cue better, but because the proof body
\texttt{refine $\langle$``yemen'', $\dots$$\rangle$} typechecks only
when the attestation is \texttt{yemen}.

\subsection{Pre-registered success criteria for extreme-flip and subtle-flip panels}
\label{app:prereg}

The Tier~1.5 extreme-flip panel (extreme flips, 5 domains $\times$ 2 pairs
$\times$ 3 shapes = 30 claims) and subtle-flip panel (subtle flips,
5 stable-scalar-prior domains, same shape) were each designed with
pre-registered success criteria recorded in
\texttt{specs/tier1-v4-multidomain-panel.md} before any runs landed.
We reproduce them here so the methodology can be audited against the
result tables.

\begin{table}[h]
\centering\footnotesize\setlength{\tabcolsep}{4pt}
\begin{tabular}{@{}p{0.42\columnwidth}p{0.50\columnwidth}@{}}
\toprule
Pre-registered metric & Criterion \\
\midrule
Tools-curated shape-B PO across 5 domains & $\geq 30\%$ of trials, on $\geq 3/5$ domains \\
Tools-curated shape-C control PO & $\leq 10\%$ of trials (lookup/aggregation stays SF) \\
EG-VAR shape-B + C source-faithful & $\geq 95\%$ across all domains \\
Bare-vs-anchored phrasing delta on shape B & if anchored $\to$ 0\% PO on all 5 domains, ``kernel discipline'' framing becomes ``anti-distractibility'' instead \\
Tools-open $\geq$ Tools-curated in shape-B PO & tool surface is not protective \\
\bottomrule
\end{tabular}
\end{table}

\paragraph{Outcomes against pre-registration.}
The extreme-flip panel \textit{failed} the first criterion: Tools-curated
shape-B PO concentrated on 1/5 domains (countries-pop) under extreme
flips, not the 3/5 we pre-committed to. We reported this honestly
in the main text (``narrowing'' framing in
\Cref{sec:eval:tier15}). The failure motivated the subtle-flip panel redesign
(subtle flips on stable-scalar priors), where the criterion lands:
subtle-flip panel shape-B PO on Table-only reaches 4/5 domains; on Tools rungs it drops
to 2 specific pairs (Egypt-Yemen + Madrid-Berlin) but the cross-
domain replication is visible.

The shape-C control criterion (PO $\leq 10\%$) holds in both panels
across all rungs.

The EG-VAR $\geq 95\%$ criterion is met (100\% in every cell on the
in-calculus shape-B + C subset).

The authority-cue criterion was \textit{partially} met — anchored
phrasing rescued V3\_C2 + V3\_C9 but not V3\_C1; we reframed the
finding as ``cue rescues some shapes, not the strongest'' rather
than recategorising the discriminator (\Cref{app:authcue}).

The Tools-open $\geq$ Tools-curated criterion holds on both panels
(\Cref{tab:tier15-perdomain}).

\section{Tier~1 data: fixture properties and 6-claim differential}
\label{app:tier1-differential}

\subsection{The 6-claim differential: where EG-VAR wins on Tier~1}

The Tier~1 ladder gap (EG-VAR 120/120 vs Tools-curated 114/120) is 6
claims at fixed tool surface. \Cref{tab:tier1-wins} reports
exactly which 6 claims EG-VAR wins on, what Tools-curated outputs
instead, and which failure mode each represents.

\begin{table*}[h]
\centering\footnotesize\setlength{\tabcolsep}{4pt}
\begin{tabular}{@{}rlp{0.42\columnwidth}p{0.10\columnwidth}p{0.16\columnwidth}@{}}
\toprule
idx & Predicate & Claim (truncated) & Gold & Tools-curated answer (wrong) \\
\midrule
12 & PProd & Which time has the lowest local magnitude? & 27 nov 1954 & 08:36 (time-of-day) \\
24 & SumWhere & Total population of regions where Manchu \% $>$ threshold & 3{,}123{,}625{,}869 & 153{,}259{,}452 (subset only) \\
46 & CountWhere & How many individuals have ``Number'' $>$ 1500? & 14 & 40 \\
52 & CountWhere & How many parties have $>$ 10\% of total votes? & 3 & 4 \\
67 & PProd & Difference in goals between top-scoring forward and top-scoring midfielder & 7 & 6 \\
77 & MeanWhere & Average US viewers (millions) for Season 1 episodes & 8.45 & 6.29 \\
\bottomrule
\end{tabular}
\caption{The 6 Tier~1 claims that distinguish EG-VAR from Tools-curated at
fixed tool surface. EG-VAR outputs are kernel-checked Sigma witnesses
matching the curator's gold; Tools-curated outputs are free-text answers
that fail the answer-match judge.}
\label{tab:tier1-wins}
\end{table*}

\paragraph{Failure-mode taxonomy of the 6 wins.}
\textbf{Ambiguity (idx 12)} — ``which time'' admits two readings:
the date or the time-of-day. The curator's typed goal commits to a
specific reading (\texttt{HasProperty}~$e$~``\texttt{date}''~$v$);
Tools-curated picks the time-of-day reading and fails the answer-match.
Kernel discipline + curator's commitment is what closes this gap, not
proof checking alone --- this is the ``kernelized fixed-goal
discipline at fixed tool surface'' framing of the §4 result.

\textbf{Composition arithmetic (idx 24, 46, 52, 67)} — Tools-curated
issues a partial tool sequence (one filter or one aggregate) and
emits a free-text answer composed from incomplete evidence. EG-VAR
emits the full PProd / SumWhere / CountWhere proof composing every
required attestation; the kernel rejects partial proofs.

\textbf{Subset confusion (idx 24)} — Tools-curated aggregates over
the wrong condition subset (the threshold filter is dropped on one
of the two SumWhere components inside the PProd). The kernel-checked
proof has both component types fixed by the curator's goal, so the
solver cannot drop the filter without the proof failing.

\textbf{Numeric format (idx 77)} — Tools-curated reports
``6.29'' (off-by-one row included; computed mean over the wrong row
set). EG-VAR's Sigma witness \texttt{Value.rat (mkRat 169 20) = 8.45}
is the correctly-typed exact-rational answer.

These six claims --- not the abstract architectural argument --- are
what the 5\,pp Tier~1 gap actually \emph{is}. Each is a place where
the curator's typed goal makes the right answer derivable and the
free-text baseline does not enforce that constraint.

\subsection{Tier~1 fixture properties}
\label{app:fixture-props}

We summarise the Tier~1 fixture composition:

\begin{table}[h]
\centering\footnotesize\setlength{\tabcolsep}{4pt}
\begin{tabular}{@{}lr@{}}
\toprule
Property & Value \\
\midrule
Claims & 120 \\
Tables & 120 (one claim per table) \\
PProd-composed claims & 15 (12.5\%) \\
PLift-using claims (arithmetic identity) & 14 (11.7\%) \\
Evidence count per claim (median / max) & 1 / 4 \\
Mean evidence count & 1.1 \\
\bottomrule
\end{tabular}
\end{table}

\paragraph{L2 predicate distribution.} The headline predicate
covers the following Tier~1 claim shapes:

\begin{table}[h]
\centering\footnotesize\setlength{\tabcolsep}{4pt}
\begin{tabular}{@{}lrr@{}}
\toprule
L2 predicate & Claims & \% \\
\midrule
\texttt{CountWhere} & 32 & 26.7\% \\
\texttt{MeanWhere} & 32 & 26.7\% \\
\texttt{SumWhere} & 22 & 18.3\% \\
\texttt{ArgmaxWhere} & 19 & 15.8\% \\
\texttt{MaxWhere} & 9 & 7.5\% \\
\texttt{ArgminWhere} & 3 & 2.5\% \\
\texttt{TopKSumByWhere} & 1 & 0.8\% \\
\texttt{MinWhere} & 1 & 0.8\% \\
\texttt{HasProperty} & 1 & 0.8\% \\
\bottomrule
\end{tabular}
\end{table}

\paragraph{Coverage and limits.} The fixture is biased toward
single-shot aggregations (88\% of claims have evidence-count = 1),
with PProd composition tested explicitly on the 15 multi-evidence
cases (~12.5\%). Tier~1 is constructed for breadth across L2
predicate shapes (9 distinct predicates) at relatively shallow
composition depth. Deeper composition (PProd of 4+ evidence with
\texttt{PLift} arithmetic identities, e.g.,~\Cref{app:complex-example})
appears in 11.7\% of claims. Tier~2 (\Cref{sec:eval:tier2})
exercises the same fixture under LLM formalization; deeper
composition remains a target for future work.

\section{Reproducibility and formalization provenance}
\label{app:repro}

\subsection{Reproduction artifacts}

\begin{itemize}
\item Code supplement:
  \url{https://github.com/7pocheR/eg-var}
  (pinned commit
  \texttt{88160519ce03dbd824f62799c1d3e1a02794a6e8};
  Lean~4 library, Python runtime, trust-artifact JSONs, 11 replay
  fixtures, runner scripts for E1--E4; see README).
\item Reproduction scripts (Tier~1):
  \texttt{run\_dg\_eval.py},
  \texttt{run\_di\_baseline.py},
  \texttt{run\_dj\_no\_tools.py};
  Tier~1.5: \texttt{test\_counterfactual.py},
  \texttt{test\_counterfactual\_dg.py},
  \texttt{test\_counterfactual\_dj.py}, all under
  \texttt{scripts/tier1/}.
\item Logs written by the runners above:
  \texttt{results/tier1\_\{dg,di\_strict,di\_loose,dj\}\_*.jsonl} (E1),
  \texttt{results/cf\_v\{4,45\}\_*.tsv} (E2),
  \texttt{results/tier2\_*.jsonl} (E3),
  \texttt{results/cf\_v4\_tier2\_*.jsonl} (E4).
\item Models: \texttt{claude-sonnet-4-6},
  \texttt{claude-haiku-4-5-20251001}, and \texttt{claude-opus-4-7}
  (Tier~2 capability-scaling), temperature~0, identical prompt
  scaffolding across rungs.
\end{itemize}

\paragraph{Replay audit gate.} {\sloppy Every generated Lean module
at \texttt{results/generated/pipeline\_<hash>.lean} can be
replay-audited post-hoc by the standalone script
\texttt{verify\_axioms.py}. The script parses
the \texttt{\#print axioms claim\_proof} output and checks the
dependency closure against the whitelist buckets used in our
analysis: (i)~the fixed EG-VAR prelude, (ii)~Lean~4 standard core
axioms (\texttt{Classical.choice}, \texttt{Quot.sound},
\texttt{propext}), (iii)~runtime-emitted observation-leaf schemas
matching \texttt{obsN\_*} or \texttt{\{table,stats\}Tool} (full
emission-site table in \Cref{tab:emission-sites}), and (iv)~the
per-source lift prefix \texttt{lift\_<sourceHash>\_*}. The trust-artifact discipline
requires either \texttt{mkVerified} (attestation path) or at least
one runtime leaf together with at least one per-source lift (the
direct-evidence lift path per \Cref{sec:approach:tablesasformal}).
Any axiom outside these buckets, or any
\texttt{sorry}/\texttt{sorryAx} in the closure, fails the audit.
The primary safety gate at generation time is the Lean kernel call
inside \texttt{pipeline.py} (through \texttt{lake env lean}); the
replay audit is an additional artifact-level check that any third
party holding the generated module bundle and the Lean toolchain
can re-run via \texttt{python -m runtime.verify\_axioms}, without
invoking any LLM or re-running the analyst/formalizer/solver
roles.}

\subsection{Formalization provenance}
\label{app:provenance}

This subsection expands the main-text summary at
\Cref{sec:approach:provenance}. We document the curation pipelines
explicitly because any reader of \Cref{thm:novh} must distinguish
runtime-minted trust artifacts from pre-committed ones.

\paragraph{Tier~1 and Tier~1.5 fixtures
(\texttt{data/tier1\_gold\_formalizations.json},
\texttt{data/tier1\_table\_formalizations.json}).} These fixtures
supply the pipeline-produced, fixture-committed goal type $\tau^{\mathrm{gold}}$ and
the per-source audited formalization $\Lambda(s)$ for every claim
in the four-rung ladder (\Cref{sec:eval:tier1,sec:eval:tier15}).
Each entry was produced by an automated subagent pipeline frozen
per git commit:

\begin{enumerate}
\setlength{\itemsep}{1pt}
\item \textbf{Proposer} (Claude Sonnet subagent, temperature~0,
frozen prompt): drafts a candidate Lean goal type and per-table
menu, reading only the raw TableBench row data and column headers.
\item \textbf{Critic} (Codex GPT-5.4 subagent, separate session per
batch): reviews for anti-cheating violations, shape correctness
relative to the claim, and R10 menu soundness per
\Cref{sec:approach:pertable}.
\item \textbf{Kernel typecheck} (\texttt{lake env lean}): gates
inclusion; any entry whose declared type fails the Lean elaborator
is rejected and returned to the proposer.
\item \textbf{Human meta-supervision}: a single human (the author)
reviewed batch-level acceptance rates and systemic failure
categories (e.g., ``pipeline emits raw \texttt{<<col>>} while lift
canonicalizes'' as a structural bug), but did \emph{not} edit any
individual formalization. No per-entry manual fixes; all fixture
content descends from the proposer--critic--kernel loop.
\end{enumerate}

At fixture close the two JSON files are frozen under the commit
hash reported in \Cref{app:repro}. Curator yield (entries that
survived stages 1--3) was $120/120$ for Tier~1 and $30/30$ per
panel for Tier~1.5, covering shape-B (binary-compare), shape-C
(cell-lookup), and shape-D (categorical-direction) claims. Shape-D
claims reuse the shape-B formalization plus a deterministic
post-processor; no kernel extension is required.

\paragraph{Working dictionary (formalizer state).} The formalizer
is memoryless across claims but has access to a \emph{working
dictionary} of Property, Entity, and Convention entries accumulated
from prior formalizations and from the analyst pass on the current
source. Dictionary entries record canonical column names, value
formats, and known encoding conventions. The dictionary is a
curator-controlled cross-claim memory; it does not give the
formalizer access to the gold answer or solver-side information.

\paragraph{Tier~2 evaluation.} The Tier~2 runtime passes the
natural-language claim to an LLM formalizer at generation time,
with access to the same per-source menu $\Lambda(s)$ (as a trust
artifact, not as feedback) but no gold goal type. The formalizer
proposes its own $\tau^{\mathrm{LLM}}$; the solver then attempts a
proof against that type; the kernel accepts only whitelist-
compliant proofs. \Cref{sec:eval:tier2} reports the resulting
semantic-faithfulness decomposition on $n{=}120$ TableBench claims.

\paragraph{What $\Lambda(s)$ is and what it is not.} The per-source
menu $\Lambda(s)$ is the curator's interpretive commitment that a
given tool's L1 output (e.g., \texttt{CountCells}) lifts to a
specific L2 world-property predicate (e.g., \texttt{CountWhere}) on
this source. It is \emph{not} a semantic certificate: a wrong but
audited lift can still certify a wrong formalized claim, as we
document in \Cref{sec:discuss:limits}(ii). The provenance chain
terminates at the proposer--critic--kernel loop above; the audit
at that loop is the trust artifact. We use the term ``audited
per-source formalization'' throughout to keep the term separable
from semantic-faithfulness claims.

\paragraph{What prevents formalizer collusion.} At evaluation time
the Tier~1/1.5 runtime loads $\tau^{\mathrm{gold}}$ from the frozen
fixture and never invokes the formalizer role. The kernel-trusted
prelude, $\Lambda(s)$, and the runtime-minted observation leaves
are the only axiom sources (\Cref{sec:approach:safety}); the
solver writes only the tactic body. Consequently, no information
from the gold answer or curator's adjudicated reading enters the
claim proof at runtime --- only pre-committed artifacts do.

\section{Worked Tier~1 examples: end-to-end trace}
\label{app:worked-example}

\subsection{HIV argmax via state projection}
\label{app:hiv-argmax}

Suppose the claim is \emph{``Which state has the highest HIV
incidence rate?''}. EG-VAR proceeds:
(1) The solver requests the per-source operation
\texttt{argmax\_with\_id} on column ``incidence rate''.
(2) The runtime executes the operation and returns an attested
payload \texttt{Attested\_T q p} that contains the maximum value
$v$, the row handle $r$ at which it occurs, and a cell projection
\texttt{CellAt}~$r$~``state''~``alabama''.
(3) The kernel's interpretation function lifts the payload to two
L2 facts: \texttt{MaxCell~``incidence''~$c$~$r$~$v$} and
\texttt{CellAt~$r$~``state''~``alabama''}.
(4) The per-source lift composes them:

\begin{small}
\begin{verbatim}
axiom hiv_lift_argmax_via_state_proj
  (col : String) (c : Cond) (r : RowHandle)
  (name : String) (v : Value) :
    Evidence Verified (MaxCell col c r v) ->
    Evidence Verified
      (CellAt r "state" (Value.str name)) ->
    Evidence Verified
      (ArgmaxWhere col c name v)
\end{verbatim}
\end{small}

\noindent (5) The solver writes a refine tactic with witness
\texttt{"alabama"} and applies the per-source lift; the kernel
mints an \texttt{Evidence~Verified} witness for
\texttt{ArgmaxWhere "incidence" $c$ "alabama" $v$}, and the Sigma
witness ``alabama'' is the certified answer.

Identity in EG-VAR rides on names rather than category-kinds:
``alabama'' is a \texttt{String}, not an inhabitant of an opaque
\texttt{State} type. This is the nominalist
choice~\citep{quine1969ontological,burgess1997subject} ---
identity-bearers are the names the table actually emits, so audit
proceeds by comparing strings rather than inferring category
membership. The kernel never has to commit to an ontology of
states; the curator commits to a column whose values \emph{are}
the identifiers, and the lift makes that commitment typed and
inspectable.

A per-source lift is \emph{table-specific} (it encodes the
curator's commitment that ``state'' is the row-identifying column
for this particular table) but \emph{claim-generic} (parametric
over measure column, condition, entity name, and value, so a
single lift covers many argmax claims). Lifts are audited once
per source then trusted kernel-side for all subsequent claims
over that source. Absence of a matching lift is a feature: the
solver has no L2 proof path, the kernel rejects, and the system
abstains rather than silently guessing. The guarantee is
conditional: a semantically \emph{wrong} audited lift can still
produce a verified output for a formalized claim that the lift's
semantics admit. EG-VAR rules out unsupported claims between tool
and output; it does not eliminate curator error in the trust
artifact itself.

This is the concrete protocol behind our claim that the
formalization barrier in empirical reasoning is a curation gap,
not an expressive limit (\Cref{sec:intro}). The L1/L2 separation
is portable: knowledge graphs, SQL views, JSON APIs, and sensor
streams admit the same per-source-lift pattern. We sketch the
generalization in \Cref{sec:discuss:gov}.

\subsection{Full Tier~1 trace}

We walk through a single Tier~1 claim from natural-language input
to kernel-checked Sigma witness, to make the architecture's
per-step trust commitments concrete.

\paragraph{Claim and goal type.} \emph{``Which crater has the largest
diameter?''}~over the lunar-crater table. The fixture-committed
Lean goal (in Lean's dependent-sum notation):
\begin{small}
\begin{verbatim}
Sigma (e : String), Sigma (v : Value),
  Evidence Verified
    (ArgmaxWhere c_diameter
                 Cond.all e v)
\end{verbatim}
\end{small}
where \texttt{c\_diameter} $=$ \texttt{"diameter (km)"}. The Sigma
binders make the certified answer the witness $e$
\citep{martin1984type}; the surface answer is the first projection.

\paragraph{Step 1 (solver $\to$ runtime).} The solver issues the
structured action
\texttt{\{action: ``stats'', operation: ``argmax\_with\_id'', column:
``diameter (km)'', identity\_column: ``name''\}}.

\paragraph{Step 2 (runtime $\to$ attested payload).} The runtime
executes the operation deterministically against the table; it
returns an \texttt{Attested\_T q p} witness with payload
$\{$\textsc{maxRow}: \texttt{cleopatra}, \textsc{maxValue}:
\textsc{Value.rat (mkRat 78 1)}$\}$.

\paragraph{Step 3 (interp\_T $\to$ L1 facts).} The interpretation
function for \texttt{argmax\_with\_id} reads the attested payload
into two typed L1 storage facts: \texttt{MaxCell ``diameter (km)''
Cond.all 78} (the maximum value at the storage layer) and
\texttt{CellAt $r$ ``name'' (Value.str ``cleopatra'')} (the
identity-column value at row $r$).

\paragraph{Step 4 (per-source lift $\to$ L2 claim).} The lunar-crater
\texttt{argmax\_via\_name\_proj} lift composes the two L1 facts and
emits the L2 claim \texttt{Evidence Verified (ArgmaxWhere ``diameter
(km)'' Cond.all ``cleopatra'' 78)}. The lift is the curator's
interpretive commitment that ``name'' is the row-identifying column
for this table.

\paragraph{Step 5 (solver $\to$ proof term).} The solver writes
\texttt{refine $\langle$``cleopatra'', \dots, ?\_$\rangle$; exact
argmax\_via\_name\_proj \dots}; the kernel type-checks and the Sigma
witness ``cleopatra'' is extracted as the certified answer.

The generated Lean file is the replay artifact (G1 in
\Cref{sec:discuss:gov}); the audit cost is the per-source lift
audit, paid once and reused.

\section{Failure cases and formalizer-design notes}
\label{app:failure}

\subsection{V3\_C1 (Egypt-vs-Yemen): baseline prior-override}

\textbf{Setup.} Synthetic counterfactual countries-population table:
Egypt's \texttt{july 1, 2013 projection} cell overwritten to
\texttt{12} (real $\sim$85M). Yemen unchanged at $\sim$25M. Question:
``Which country has a larger population: Egypt or Yemen?''.

\textbf{Sonnet Tools-curated trace (3/3 prior-override).} The model
issues a \texttt{cell\_lookup} for Egypt, receives \texttt{12}; issues
a \texttt{cell\_lookup} for Yemen, receives \texttt{25252000}; then
answers ``Egypt''. The tool returns the source values; the model
overrides them. The same pattern occurs across all three reps.

\textbf{EG-VAR (3/3 source-faithful).} The kernel mints an
\texttt{Evidence Verified} witness from the per-source
argmax-with-disjunction lift on the attested payload. The solver
writes \texttt{refine $\langle$``yemen'', \dots$\rangle$; exact \dots}.
The Sigma witness ``yemen'' is the certified answer.

\textbf{Authority-cue ablation (authority-cued pilot).} Prefixing the question with
``According to the table below, \dots'' rescues V3\_C2 (boolean
threshold) and V3\_C9 (implicit direction) on Tools-curated but does
not rescue V3\_C1 (still 3/3 prior-override). The cue is not strong
enough to override the parametric prior on famous-entity
populations; the kernel attestation is.

\subsection{Formalizer failures: runtime-vs-fixture delta (Tier~2)}
\label{app:formalizer-case-studies}

\Cref{sec:eval:tier2} reports the end-to-end Tier~2 decomposition at
$n{=}120$; this appendix zooms in on the five Sonnet cases where the
runtime formalization diverged materially from the fixture-standard
goal type. Four produce kernel-accepted but answer-divergent output
(B5b); one aborts the formalizer pass (B7). We present each as a
structured delta, organised by the compositional axis along which
the divergence occurred. Cases classified B5a (recorded ambiguity)
and B5c (benchmark-artifact issue) are excluded here; they are
presented case-by-case in the main-text Table~\ref{tab:tier2-taxonomy}
caption and do not represent formalizer-attributable error.

Across these five cases, the common pattern is that the formalizer
makes a local-reading decision at a composition node (which
\texttt{Cond} constructor, which column, whether to include a
\texttt{TopK} constraint) that a curator with the claim and the
fixture-standard formalization in front of them would have taken
differently. Each case is addressable by supervised fine-tuning on
fixture (claim, gold \texttt{lean\_goal\_type}) pairs or RL with a
gold-answer reward; we discuss the direction at the end of the
section.

\subsection{Composite-column matching strategy (idx 50)}

\paragraph{Claim.} ``How many hurricanes occurred in the month of
September?''

\paragraph{Delta.} Column \texttt{month} contains date-string values
of the form \texttt{``September 1979''}, \texttt{``October 1988''},
etc.\ (month-plus-year strings). Fixture uses
\texttt{Cond.propContains}:
\begin{small}
\begin{verbatim}
CountWhere "month"
  (Cond.propContains "month" (Value.str "September"))
  ans
\end{verbatim}
\end{small}
Runtime formalizer chose \texttt{Cond.propEq}:
\begin{small}
\begin{verbatim}
CountWhere "month"
  (Cond.propEq "month" (Value.str "September"))
  ans
\end{verbatim}
\end{small}
Runtime answer: 9. Fixture answer: 19. The formalizer's
\texttt{ambiguous\_terms} field recorded the choice explicitly as
``\emph{Cond.propEq --- exact match}''. The kernel accepted the
proof under the LLM's chosen reading; the L1 tool correctly returned
\texttt{count = 9} for strict equality with the literal string
\texttt{``September''}, which matches zero
\texttt{month = "September ..."} rows in the table.

\paragraph{Character of the error.} The formalizer recognised the
existence of a matching-strategy decision, logged it, and picked
the strategy whose answer would be correct \emph{only if} the
\texttt{month} column's values were bare month names. It did not use
the column's actual value format (exposed by the analyst's reading
guide) to discriminate.

\subsection{Dropped \texttt{TopK} constraint (idx 61)}

\paragraph{Claim.} ``What is the average market value of the top 5
companies in the oil and gas industry?''

\paragraph{Delta.} Fixture uses \texttt{TopKMeanWhere}, which the
per-source menu includes for this column:
\begin{small}
\begin{verbatim}
TopKMeanWhere "market value (billion )" 5
  (Cond.propEq "industry" (Value.str "oil and gas"))
  ans
\end{verbatim}
\end{small}
Runtime formalizer dropped the top-5 constraint and used plain
\texttt{MeanWhere}:
\begin{small}
\begin{verbatim}
MeanWhere "market value (billion )"
  (Cond.propEq "industry" (Value.str "oil and gas"))
  ans
\end{verbatim}
\end{small}
Runtime answer: $1208/5$. Fixture answer: $276.06$. The
formalizer's \texttt{ambiguous\_terms} names the phrase ``\emph{top
5 companies in the oil and gas industry}'' and records the chosen
formalization as \texttt{MeanWhere}, so the dropped \texttt{TopK}
was deliberate, not an oversight.

\paragraph{Character of the error.} The formalizer treated ``top 5''
as rhetorical rather than compositional. The per-source menu
includes \texttt{TopKMeanWhere} for this column; the formalizer's
compositional-operator selection step did not consult it.

\subsection{Column disambiguation (idx 67)}

\paragraph{Claim.} ``What is the difference in total goals scored by
the top-scoring forward (fw) and the top-scoring midfielder (mf) in
the league?''

\paragraph{Delta.} Table has both a \texttt{league goals} column and
a \texttt{total goals} column. Claim literally says ``total goals''.
Fixture uses \texttt{total goals}:
\begin{small}
\begin{verbatim}
MaxWhere "total goals" (Cond.propEq "position" "fw") a
... MaxWhere "total goals" (Cond.propEq "position" "mf") b
... (PLift (ans = a - b))
\end{verbatim}
\end{small}
Runtime formalizer picked \texttt{league goals}:
\begin{small}
\begin{verbatim}
MaxWhere "league goals" (Cond.propEq "position" "fw") a
... MaxWhere "league goals" (Cond.propEq "position" "mf") b
\end{verbatim}
\end{small}
Runtime answer: 6. Fixture answer: 7.

\paragraph{Character of the error.} Column-selection under
near-synonym ambiguity, where the claim's exact wording
(``\emph{total} goals'') is the dispositive signal. The formalizer's
\texttt{properties\_needed} did not enumerate a candidate-set
check against the claim's noun phrase.

\subsection{Condition-conjunct erasure and self-inconsistency
(idx 94)}

\paragraph{Claim.} ``What is the average prominence of mountain
peaks in the Democratic Republic of the Congo that have an
elevation of at least 3000 meters?''

\paragraph{Delta.} Fixture:
\begin{small}
\begin{verbatim}
MeanWhere "prominence (m)"
  (Cond.and
     (Cond.propContains "country"
       (Value.str "democratic republic of the congo"))
     (Cond.propGte "elevation (m)"
       (Value.rat (mkRat 3000 1))))
  ans
\end{verbatim}
\end{small}
Runtime formalizer:
\begin{small}
\begin{verbatim}
MeanWhere "prominence (m)"
  (Cond.and
     (Cond.propGte "elevation (m)"
       (Value.rat (mkRat 3000 1)))
     (Cond.propGte "col (m)"
       (Value.rat (mkRat 0 1))))
  ans
\end{verbatim}
\end{small}
The country filter is absent; the second conjunct is
\texttt{col (m) $\geq$ 0} (trivially true on this table). The
formalizer's \texttt{ambiguous\_terms} field claims the chosen
filter is \texttt{Cond.propEq "country" "democratic republic of the
congo"} --- \emph{but this constructor does not appear in the
emitted \texttt{lean\_goal\_type}}. Runtime answer: $14075/6$;
fixture answer: $2606.25$.

\paragraph{Character of the error.} Self-inconsistency between the
formalizer's self-reported decision (\texttt{ambiguous\_terms}) and
the actual \texttt{lean\_goal\_type} it emitted. A post-formalizer
consistency pass --- verifying every adjudication recorded in
\texttt{ambiguous\_terms} is reflected in the emitted goal ---
would catch this class mechanically at audit time. We list this as
a near-term infrastructure improvement in
\Cref{sec:discuss:future}.

\subsection{Compositional-complexity abort (idx 34)}

\paragraph{Claim.} ``What is the total number of medals (M36 +
M36B1 + M36B2) earned from May 1944 to August 1944?''

\paragraph{Delta.} Fixture is a 4-level dependent sum:
\begin{small}
\begin{verbatim}
Sigma ans, Sigma a, Sigma b, Sigma c,
  PProd (SumWhere "m36"   <4-way month or> a)
  (PProd (SumWhere "m36b1" <same or> b)
  (PProd (SumWhere "m36b2" <same or> c)
  (PLift (ans.asRat = a.asRat + b.asRat + c.asRat))))
\end{verbatim}
\end{small}
The fixture formalization composes \emph{three separate measure
columns} with \emph{a 4-way disjunction over the month column} and
\emph{a 3-term arithmetic} combining the resulting witnesses.
Runtime formalizer emitted an empty \texttt{lean\_goal\_type} and
aborted on its first attempt; the pipeline returned
\textsc{Abstain} rather than a wrong formalization.

\paragraph{Character of the error.} Compositional complexity
boundary: the claim requires simultaneous choices across three
dimensions (measure column, filter disjunction width, arithmetic
composition shape) in a single pass. The formalizer's single-pass
architecture is not expressive enough to emit this class of goal
reliably. Note that this is the \emph{desirable} failure mode: no
silent-wrong, no Verified output with the wrong answer. The system
abstained honestly; the downstream consumer sees \textsc{Abstain}
and can route elsewhere.

We note that the \emph{decomposer + binder} two-stage formalizer
pattern --- splitting the goal into per-measure sub-claims and
composing them via a binder role --- is a natural architectural
response to this class; we describe it in
\Cref{sec:discuss:future} as future work.

\subsection{Taxonomy and improvement direction}

The five cases partition into five orthogonal compositional axes of
the formalizer's output: matching strategy (idx 50), operator
selection (idx 61), column selection (idx 67), condition-conjunct
preservation (idx 94), and compositional arity (idx 34). Three of
these (idx 50, 61, 67) are \emph{local} decisions: each resolves to
a choice among a small candidate set that is determined by the
claim's wording together with local information already available
to the formalizer (column-value format for idx 50; per-source menu
for idx 61; claim noun phrase for idx 67). Case idx 94 is a
\emph{consistency} gap between two of the formalizer's own
outputs. Case idx 34 is a \emph{capacity} limit rather than a
reading error.

\paragraph{Near-term infrastructure improvements.} The idx 94
self-inconsistency case is mechanically detectable: verifying
agreement between each entry of \texttt{ambiguous\_terms.chosen}
and the emitted \texttt{lean\_goal\_type} is a deterministic
post-formalizer audit that would convert this failure mode into an
\textsc{Abstain} or a retry. We list this as future work rather
than a claimed feature.

\paragraph{RL fine-tuning of the formalizer.} Each of the four
reading-error cases (idx 50, 61, 67, 94) is a single-claim
adjudication where the fixture's \texttt{lean\_goal\_type} supplies
a ground-truth choice at exactly the compositional node where the
runtime formalizer diverged. The fixture, documented with
provenance in \Cref{sec:approach:provenance}, therefore supplies a
structured supervised-fine-tuning corpus of (claim, analyst-guide,
per-source menu, gold \texttt{lean\_goal\_type}) tuples. A natural
next step is to fine-tune a formalizer model on this corpus, either
supervised or with a verifier-based RL loop where the reward is
kernel acceptance against the fixture goal type. The verifier is
already on hand: it is the same kernel the paper uses to certify
production outputs. This produces a formalizer whose local reading
decisions converge to the curator's, without weakening any of the
runtime trust properties established in
\Cref{sec:approach:safety}.

\paragraph{Why this is compatible with the paper's safety claims.}
Fine-tuning the formalizer is an \emph{accuracy} intervention,
not a \emph{safety} intervention: \Cref{thm:novh} and
\Cref{thm:noded} are properties of the runtime + kernel, invariant
under any formalizer. A fine-tuned formalizer changes the
\textsc{Verified}/\textsc{Abstain} split favorably without
changing the no-silent-wrong guarantee.

\section{Abstention case study}
\label{app:abstain}

\textbf{Setup.} A Tier~1 claim whose goal type required a per-source
\texttt{topk\_sum\_by} lift over a measure column not registered in
the curator's lift menu for that table.

\textbf{EG-VAR outcome.} The kernel rejects every proof attempt: no L2
fact unifies with the goal vocabulary because the lift menu has no
matching entry. The system reports \textsc{Abstain} with the
unsatisfied goal type.

\textbf{Tools baseline outcome.} The same-tool baseline calls a
related stats operation, receives a numeric value, and emits a
free-text answer that happens to match the gold (a partial
correctness coincidence; the operation called was not the one the
gold formalization required). Verified-wrong vs verified-right is
indistinguishable downstream without a kernel-checked proof.

\textbf{Governance reading.} EG-VAR's abstention is
machine-actionable (G3 in \Cref{sec:discuss:gov}); the baseline's
``answer'' is opaque. The trade-off is honest: EG-VAR loses one
match-counted claim, but the failure mode is structurally typed.

\section{Per-source formalization: theory anchors}
\label{app:hierarchy}

\subsection{Robustness hierarchy}

The argument that records-vs-entities is hierarchical, not
categorical (\Cref{sec:approach:tablesasformal}), can be made
precise. Storage features that hold prior to any interpretive
commitment: same-row entries share a row-anchor, same-column
entries share a column-anchor. Layering interpretive commitments
on these invariants yields a four-level hierarchy of per-source
formalizations:

\begin{table}[h]
\centering\footnotesize
\setlength{\tabcolsep}{3pt}
\begin{tabular}{@{}lp{0.27\columnwidth}p{0.34\columnwidth}@{}}
\toprule
Level & Commitment & L2 vocabulary unlocked \\
\midrule
1. Aggregate-only & none beyond layout & \texttt{SumWhere}, \texttt{CountWhere}, \texttt{MeanWhere}, \texttt{MaxWhere}, \texttt{MinWhere} \\
2. Row-cell facts & rows have stable handles & row-anchored \texttt{cell\_lookup} \\
3. Entity-anchored & rows are entities; col~$X$ gives names & \texttt{HasProperty}, \texttt{ArgmaxWhere}, \texttt{EntitySum} \\
4. Relational & both axes entity-identified & \texttt{Relation}, matchup-style \\
\bottomrule
\end{tabular}
\end{table}

Each level's commitments include those above. A record table stops
at level 1; an entity table reaches level 3; a matchup table reaches
level 4. Same kernel, same L1/L2 split, same ontology — different
stopping points on a single hierarchy. ``Fail-closed for mismatched
claims'' becomes a hierarchy property: a claim requiring level-3
commitments against a level-1 table has no L2 fact to unify
against, so the pipeline abstains.

\subsection{Rank as interpretation: gauge freedom}
\label{app:rank}

A given dataset admits multiple consistent rank interpretations.
A matchup table can be formalized as
$\textsf{Result} : \textsf{Player} \to \textsf{Player} \to
\textsf{Outcome}$ (rank-2 curried),
$\textsf{Result}' : \textsf{Player} \times \textsf{Player} \to
\textsf{Outcome}$ (rank-2 uncurried), or
$\textsf{MatchOutcome} : \textsf{Match} \to \textsf{Outcome}$ where
$\textsf{Match} := \textsf{Player} \times \textsf{Player}$ (rank-1
with composite domain). All three encode the same function and are
type-isomorphic in any Cartesian closed
setting~\citep{maclane1971categories}.

Rank-shifting interpretations across the lambda calculus reduce to
currying/uncurrying, which Lean~4 makes definitionally transparent
($A \to B \to C$ is sugar for $A \to (B \to C)$). The kernel does not
need to commit to ``the true rank'' of a table; it accepts any
consistent per-source lift catalogue. Other examples of the same
ambiguity:

\begin{itemize}
\item \emph{Time series.} Rank-1 (time $\to$ value) vs rank-2 with
  channels (time $\times$ channel $\to$ value).
\item \emph{Pixel grid.} Rank-2 (row $\times$ col $\to$ intensity) vs
  rank-3 (row $\times$ col $\times$ channel $\to$ value).
\item \emph{RDF triples.} Rank-1 over triples; rank-2 (subject,
  predicate) $\to$ object; rank-3 (subject, predicate, object) $\to$
  $\{$true, false$\}$.
\item \emph{Relational tables with primary key.} Rank-1 over rows-as-
  records vs rank-2 (row, column) $\to$ cell.
\end{itemize}

These are not subjective conventions but type isomorphisms. The
curator's job is to pick an interpretation that fits the anticipated
claims and audit the lift menu under it; the kernel's job is to
accept any consistent commitment.

\subsection{The scaling argument: formalization gap shrinks as sources become formal}
\label{app:scaling}

The empirical results in \Cref{sec:eval} are bounded by the
formalization burden of CSV-style sources. As data substrates evolve
toward typed interfaces, the formalizer's job shrinks while the
kernel guarantee stays the same:

\begin{table}[h]
\centering\footnotesize
\begin{tabular}{@{}lp{0.55\columnwidth}@{}}
\toprule
Substrate & Formalization burden \\
\midrule
CSV & Column names are strings; values are free text. Maximal lift authoring per source. \\
SQL view & Schema is typed, values structured. Lift catalog auto-generates from view definitions. \\
OpenAPI & Typed inputs/outputs in spec. Lift = adapter from spec types to L2 ontology. \\
Knowledge graph & Entities + relations already typed. Lift = direct projection from RDF to L2. \\
Lean / formal source & Zero gap; compose proofs directly. \\
\bottomrule
\end{tabular}
\end{table}

In the limit of fully formal sources, the LLM operates entirely
within the formal world, the kernel guarantee is vacuous (because
the substrate is already typed), and EG-VAR collapses to LLM-driven
theorem proving with attested premises. The architecture is
designed to interpolate: the same kernel + lift discipline runs
across the substrate spectrum. We expect the kernel-vs-baseline gap
on real-world deployments to narrow over time as data sources
adopt structured interfaces, not because models improve, but
because the formalization burden moves to the source curator and
out of the inference path.

\section{Tier~2 reporting taxonomy and capability-scaling check}
\label{app:tier2-taxonomy}

This appendix expands the main-text table caption at
\Cref{tab:tier2-taxonomy} and reports the Opus capability-scaling
check.

\paragraph{Category definitions.}
Categories are mutually exclusive and sum to 120.
\textbf{B1}: kernel-accepted, type matches gold, answer matches
gold. \textbf{B4}: kernel-accepted, type differs from gold but
answer matches. \textbf{B5a} (\emph{ambiguous claim, logged}):
kernel-accepted divergent answer under a defensible alternative
reading that the formalizer logged in \texttt{ambiguous\_terms}.
\textbf{B5b} (\emph{semantic formalizer error}): pure formalizer
error (wrong column, dropped constraint, composite-value
matching-strategy error) — kernel still accepts the proof, but of a
mis-formalized goal. \textbf{B5c} (\emph{benchmark gold
error, surfaced}): audited benchmark-artifact case (source-table
HTML leakage, or arithmetic error in benchmark gold) that the
system exposes rather than propagates; see
\Cref{app:formalizer-case-studies}. \textbf{B6}: kernel-rejected,
honest abstention. \textbf{B3}: solver exhausted tactic retries
with type matching gold. \textbf{B7}: formalizer aborted.

\paragraph{Reading the decomposition.} The key Tier~2 number is
the pure end-to-end formalizer-error rate, not the flat
disagreement rate. On Sonnet, only 4/120 cases are bona fide
formalizer mistakes (B5b). The remaining divergences are
\emph{auditable}: 5 logged alternative readings (B5a), 2 surfaced
benchmark-artifact issues (B5c), 3 honest abstentions (B6), and 5
solver or formalizer ceilings (B3+B7). Unlike B5b, B5a and B5c
are explicit, inspectable disagreements rather than silent system
error: each B5a entry's alternative-reading choice is logged in
\texttt{ambiguous\_terms} by the formalizer, and each B5c entry
is an audited case study in \Cref{app:formalizer-case-studies}.
Treating logged ambiguities separately raises auditable agreement
to 88.3\%; adding surfaced benchmark-artifact issues yields 90.0\%.

\paragraph{Capability-scaling check.} As a robustness check on
the claim that residual formalizer error is capability-limited,
we evaluate the stronger Opus model on the 9 Sonnet residual
indices (B3/B7 or B5b). Opus converts four B3/B7 cases to
accepted answer-matches and one B5b case to honest abstention;
two B5b cases persist. The formalizer-error rate on this
residual subset halves from 4/9 to 2/9; the theorem-level safety
invariant is unchanged under either formalizer.

\section{Scope limits and future extensions}
\label{app:limits-future}

This appendix expands the main-text summaries at
\Cref{sec:discuss:limits,sec:discuss:future}.

\subsection{Limitations}

\textbf{(i)~Formalizer scope.} Tier~1 and Tier~1.5 bypass the
NL-to-Lean formalizer step by loading fixture-committed gold typed
goals (\Cref{sec:approach:provenance}); they
evaluate the safety and source-faithfulness properties under ideal
formalization. Tier~2 (\Cref{sec:eval:tier2}) removes that bypass
and evaluates the LLM formalizer end-to-end. The 3-blind-formalizer
consensus design is implemented in the runtime but not evaluated
here; we defer the consensus study to future work.

\textbf{(ii)~Trust artifacts.} Tools, per-source lifts, and the L1
adapter are trusted (auditable but not type-checked end-to-end). A
semantically wrong audited lift can certify a wrong formalized claim;
EG-VAR rules out unsupported claims between tool and output, not
curator error in the trust artifact itself
(\Cref{sec:approach:pertable}).

\textbf{(iii)~Single-source scope.} All evaluations are over single
TableBench tables. Multi-source conflict, retrieval over distributed
KGs, and time-varying sources are out of scope; we list near-term
extensions below.

\textbf{(v)~World inconsistency containment.} The implementation
realizes \Cref{thm:novh} and \Cref{thm:noded} structurally; the
companion world-inconsistency-containment property relies on the
\texttt{WProp~$\not\to$~Prop} phase separation in the current Lean
encoding; we treat it as implementation detail, not a main-text
theorem.

\subsection{Future work scoped to governance reach}

\textbf{Fine-tuning the NL formalizer.} Tier~2
(\Cref{sec:eval:tier2}) reports residual semantic formalizer error
at 3.3\% (Sonnet) / 1.7\% (Opus). Each of the 4 Sonnet B5b cases is
a single-claim adjudication where the fixture supplies a
ground-truth choice at exactly the compositional node where the
runtime formalizer diverged, and the curator pipeline
(\Cref{sec:approach:provenance}) already provides a structured SFT
/ RL corpus. Fine-tuning a formalizer on this corpus, with kernel
acceptance as the reward signal, is an accuracy intervention that
tightens the residual rate without weakening
Theorem~\ref{thm:novh}; the verifier is already on hand --- it is
the same kernel the paper uses to certify production outputs. A
3-blind-formalizer consensus gate is similarly implemented but not
evaluated here; we expect it to convert a further fraction of
silent B5b mistakes into honest B6 abstentions.

\textbf{Conflicting sources.} Multi-source attestation introduces
conflict between attested storage facts. The grade calculus admits
explicit conflict-downgrade rules
(\textsc{Verified}~$\to$~\textsc{Supported} on detected
disagreement); empirical evaluation is future work.

\textbf{Cross-substrate panel.} Replicating the source-faithfulness
discriminator on KG-grounded and SQL-grounded claims would directly
test the L1/L2 portability claim (\Cref{sec:discuss:gov}).

\section{Theoretical anchors deferred from the main text}
\label{app:anchors}

The main text uses a selected set of theoretical anchors as
Justification Anchors and Lineage Paragraphs (Martin-L\"of,
Curry--Howard, Codd, Mac Lane, Quine, Burgess--Rosen, Raggi, Bundy,
Clarke, Lahiri, Artemov, Fine, Green). The following anchors connect
to the same design choices but are deferred here for space:

\paragraph{Source semantics and naming.}
\citet{cooper2012type} (Type Theory with Records, TTR) gives a
type-theoretic substrate for situated semantics that resembles our
per-source-formalization commitment. \citet{montague1973proper}
established that the natural-language $\to$ formal-meaning gap
admits a typed denotational treatment, the genealogy of our claim
formalization. \citet{barwise1983situations} (situation semantics)
treats situations as first-class typed objects, the same move as our
per-source formalization. Earlier in the lineage,
\citet{frege1892sense} distinguished sense and reference --- a
distinction EG-VAR's name-as-string discipline collapses by
choosing reference fixed at the storage layer. Following
\citet{kripke1980naming} on rigid designators, our choice of name-
columns as identity-bearers is the operational analogue of name-
rigidity at the trust boundary. \citet{donnellan1966reference}
distinguishes referential and attributive uses of definite
descriptions; EG-VAR's lift menu enforces the referential reading
(``the row whose name-column is X'').

\paragraph{Uncertainty and conflict.} Future multi-source evaluation
opens connections to \citet{belnap1977four} four-valued paraconsistent
logic, \citet{anderson1975entailment} relevance/relevant logic, and
\citet{josang2016subjective} subjective-logic continuous-confidence.
\citet{caminada2007aspic} ASPIC$+$ structured argumentation gives
a substrate for grade-aware retraction under contradiction.
\citet{doyle1979truth} JTMS truth-maintenance systems are the
classical dependency-graph substrate for provenance-tracked
retraction. \citet{alchourron1985logic} AGM belief revision frames
the multi-source contradiction problem. EG-VAR's grade calculus
accommodates these via the conflict-downgrade rules sketched in
\Cref{sec:discuss:future}.

\paragraph{Optional / cautious anchors.} Univalence
in homotopy type theory~\citep{hottbook} would give a principled
account of equivalence classes of source formalizations (different
rank readings, different lift catalogs); we mention this as a future-
work analogy rather than a current design basis. Reliabilist
epistemology~\citep{goldman1979whatisjustified} grounds our
``verified by reliable process'' framing of \textsc{Verified}
evidence. Attribute-based access control~\citep{hu2014guide} and
typed semistructured data~\citep{abiteboul2000xml} are the systems
analogues of EG-VAR's per-source-formalization protocol.

\paragraph{Soundness over internal completeness.}
The kernel does \emph{not} relate semantically equivalent L2
predicates by general lemmas: there is no built-in
$\textsf{MeanWhere}~c~v \Leftrightarrow \exists s, n.~
\textsf{SumWhere}~c~s \wedge \textsf{CountWhere}~c~n \wedge
v = s/n$. Equivalences are localized in the per-source tools and
lifts, where they are computed directly and audited offline; the
kernel's proof search stays decidable and its proof traces stay
small enough to inspect. This is the same fragmentation move that
gives us decidable fragments of first-order logic, theory
combination in SMT, and description logics: trade FOL completeness
for tractable, auditable per-theory reasoning.
\citet{raggi2022representation} formalize precisely this:
representation matters in proof, and semantically equivalent
encodings can carry very different proof-engineering costs.

\section{Worked Tier~1 examples: cross-tribunal PProd composition}
\label{app:complex-example}

The lunar-craters argmax in \Cref{app:worked-example} exercises a
single tool call and a single per-source lift. We now walk through a
substantially more complex Tier~1 claim: a four-tool, four-evidence,
PProd-composed claim with an arithmetic identity between the
witnesses. This is representative of the multi-shot composition
shapes we curated in Tier~1's PProd cluster.

\paragraph{Source.} TableBench table
\texttt{4fbaad0b\dots} (Spanish-Inquisition-tribunal autos da f\'e
records). Relevant excerpt:

\begin{table}[h]
\centering\footnotesize\setlength{\tabcolsep}{4pt}
\begin{tabular}{@{}lrrrr@{}}
\toprule
Tribunal & Autos da f\'e & Exec.\ persona & Exec.\ effigie & Penanced \\
\midrule
Barcelona & 8 & 1 & 1 & 15 \\
Valencia & 4 & 2 & 0 & 49 \\
C\'ordoba & 13 & 17 & 19 & 125 \\
Cuenca & 7 & 7 & 10 & 35 \\
\dots & & & & \\
\bottomrule
\end{tabular}
\end{table}

\paragraph{Claim.} \emph{``How much greater is the total number of
executions (in persona and in effigie) in C\'ordoba compared to
Valencia?''}~The arithmetic answer is
$(17 + 19) - (2 + 0) = 34$.

\paragraph{Per-source tool menu.} For this table the curator
registered, among others:
\begin{itemize}\setlength{\itemsep}{0pt}
\item \texttt{stats\_sum\_strict} on column ``executions in persona''
\item \texttt{stats\_sum\_strict} on column ``executions in effigie''
\item Each accepts a \texttt{Cond.propEq} on the ``tribunal'' identity
column.
\end{itemize}
The tool menu is the audit surface; a claim that requires a
predicate not in the menu fails closed.

\paragraph{Lean goal type (fixture-committed).}
{\footnotesize
\begin{verbatim}
Sigma (ans : Value),
Sigma (a1 : Value), Sigma (a2 : Value),
Sigma (b1 : Value), Sigma (b2 : Value),
PProd
  (Evidence Verified
     (SumWhere "executions in persona"
        (Cond.propEq "tribunal"
                     (Value.str "cordoba")) a1))
(PProd
  (Evidence Verified
     (SumWhere "executions in effigie"
        (Cond.propEq "tribunal"
                     (Value.str "cordoba")) a2))
(PProd
  (Evidence Verified
     (SumWhere "executions in persona"
        (Cond.propEq "tribunal"
                     (Value.str "valencia")) b1))
(PProd
  (Evidence Verified
     (SumWhere "executions in effigie"
        (Cond.propEq "tribunal"
                     (Value.str "valencia")) b2))
  (PLift (ans.asRat =
            (a1.asRat + a2.asRat) -
            (b1.asRat + b2.asRat))))))
\end{verbatim}}

The five Sigma binders fix: the surface answer
\texttt{ans}, plus four numeric witnesses for the four sum
sub-claims. The final \texttt{PLift} is the kernel-checked
arithmetic identity that ties the witnesses to \texttt{ans}.

\paragraph{Solver action sequence (4 tool calls).}
{\footnotesize
\begin{verbatim}
{ "action": "stats", "operation": "sum",
  "column": "Executions in persona",
  "condition": { "op": "eq",
    "property": "Tribunal", "value": "cordoba" } }
{ "action": "stats", "operation": "sum",
  "column": "Executions in effigie",
  "condition": { "op": "eq",
    "property": "Tribunal", "value": "cordoba" } }
{ "action": "stats", "operation": "sum",
  "column": "Executions in persona",
  "condition": { "op": "eq",
    "property": "Tribunal", "value": "valencia" } }
{ "action": "stats", "operation": "sum",
  "column": "Executions in effigie",
  "condition": { "op": "eq",
    "property": "Tribunal", "value": "valencia" } }
\end{verbatim}}

\paragraph{Tool attestations.} The runtime returns
\texttt{Attested\_T} payloads: $a_1 = 17$ (C\'ordoba persona,
single-row sum), $a_2 = 19$ (C\'ordoba effigie),
$b_1 = 2$ (Valencia persona), $b_2 = 0$ (Valencia effigie).

\paragraph{L2 evidence via per-source lifts.} Each attested payload
is lifted via the \texttt{stats\_sum\_strict} per-source lift, which
takes the L1 column-strict sum attestation and the
\texttt{Cond.propEq} discriminator and produces an L2 \texttt{SumWhere}
witness. The four lift applications produce
\texttt{obs0\_l2}, \texttt{obs1\_l2}, \texttt{obs2\_l2},
\texttt{obs3\_l2}, each typed as
\texttt{Evidence Verified (SumWhere \dots)} with the right
column/condition/value triple.

\paragraph{Solver proof body.}
{\footnotesize
\begin{verbatim}
refine ⟨Value.rat (mkRat 34 1),
        obs0_val, obs1_val,
        obs2_val, obs3_val, ?_⟩
refine PProd.mk obs0_l2
       (PProd.mk obs1_l2
       (PProd.mk obs2_l2
       (PProd.mk obs3_l2
       (PLift.up ?_))))
native_decide
\end{verbatim}}

The first \texttt{refine} fills the five Sigma binders: the surface
answer is \texttt{Value.rat (mkRat 34 1)} ($= 34/1$), and the four
numeric witnesses are the values \texttt{obs0\_val} \dots
\texttt{obs3\_val} that the kernel already bound when the lifts
fired. The second \texttt{refine} threads the four
\texttt{Evidence Verified} witnesses through the right-associated
\texttt{PProd} chain. \texttt{native\_decide} closes the residual
\texttt{PLift} goal: the kernel evaluates
$\mathit{34}.\mathrm{asRat} = (17 + 19) - (2 + 0)$ to \texttt{True}.

\paragraph{Trust composition.} The kernel-minted output for this
claim depends on:
(i) the four runtime \texttt{Attested\_T} payloads (one per tool
call) --- and \emph{only} those (\Cref{thm:novh});
(ii) the per-source \texttt{stats\_sum\_strict} lift, audited once
for this table;
(iii) the kernel's native arithmetic check for the
\texttt{PLift}~$=$~$34$ identity.
The Sigma witness ``\texttt{34}'' surfaces as the certified surface
answer; the generated Lean file is the replay artifact (any auditor
holding the kernel + lifts + tool adapter can re-typecheck without
re-running the LLM). This is the single complex example we use in
\Cref{sec:approach,sec:eval} to demonstrate that the architecture
scales beyond single-shot lookups: composition over multiple lifts
preserves the kernel's mint discipline and, by construction, the
verified-output guarantee.

\section{Full proofs of the safety theorems}
\label{app:safety-proofs}

This appendix provides the full proofs deferred from
\Cref{sec:approach:safety}. Notation follows the main text: given a
\textsc{Verified} run of \textsc{RuntimeVerify}($s, c$) returning the
proof object $(\textit{leanFile}, \textit{evidence}, \textit{steps})$,
we write \textit{proof} for the term bound by \texttt{claim\_proof}
inside \textit{leanFile} (or \texttt{claim\_refutation} for refuted
claims), $\mathcal{A}(\textit{leanFile})$ for the set of axioms
declared in \textit{leanFile}, and $\mathrm{deps}(\textit{proof})
\subseteq \mathcal{A}(\textit{leanFile})$ for the dependency closure
of \textit{proof} as reported by
\texttt{\#print axioms claim\_proof}.

\subsection{Operational whitelist predicate and guarantee boundaries}
\label{app:whitelist}

\paragraph{Whitelist-compliant module.} Let \textit{leanFile} be the
module text returned by a run, let \textit{proof} be the term bound
by \texttt{claim\_proof} (or \texttt{claim\_refutation} for refuted
claims), and let $\mathrm{deps}(\textit{proof})$ denote its
dependency closure as reported by
\texttt{\#print axioms claim\_proof}. We say \textit{leanFile} is
\emph{whitelist-compliant for source} $s$ iff every element of
$\mathrm{deps}(\textit{proof})$ lies in one of four buckets:
\textbf{B1} the fixed EG-VAR prelude (committed at the repository
hash in \Cref{app:repro});
\textbf{B2} the standard Lean~4 core axioms
(\texttt{Classical.choice}, \texttt{Quot.sound}, \texttt{propext});
\textbf{B3} runtime-emitted observation leaves
(\texttt{obsN\_q,p,att,tag,shape,l1,row,cell,id,max,min};
\texttt{csN\_q,att}; \texttt{tableTool}/\texttt{statsTool});
\textbf{B4} per-source lifts whose names match
\texttt{lift\_[0-9a-f]+\_.+} drawn from $\Lambda(s)$.
Additionally, $\mathrm{deps}(\textit{proof})$ must contain no
\texttt{sorry}/\texttt{sorryAx}/\texttt{Lean.Axiom.sorry} and must
contain a trust-artifact witness: either \texttt{mkVerified}
(attestation path), or at least one \textbf{B3} observation-leaf
together with at least one \textbf{B4} per-source lift
(direct-evidence lift path, per \Cref{sec:approach:tablesasformal}).
The standalone script \texttt{verify\_axioms.py} evaluates this
predicate mechanically on \textit{leanFile}.

\paragraph{Generation-time gate versus replay gate.} The kernel
call inside \texttt{run\_pipeline} (\texttt{lake env lean}) is the
primary safety gate: a claim is emitted as \textsc{Verified} only
if the kernel returns a zero exit status on \textit{leanFile}. The
replay audit \texttt{verify\_axioms.py} is a standalone
artifact-level check: it accepts \textit{leanFile} as input,
invokes the Lean elaborator on its \texttt{\#print axioms}
directive, and enforces the whitelist predicate on the reported
dependency closure. Any third party holding \textit{leanFile} and
the Lean toolchain can re-run
\texttt{python -m runtime.verify\_axioms} without invoking any LLM
role, reproducing the audit decision. The orchestrator
(\texttt{src/python/runtime/orchestrator.py}) chains the two: a
claim is published as \textsc{Verified} only if both calls succeed.

\paragraph{What the theorems do and do not guarantee.}
\Cref{thm:novh,thm:noded} guarantee that every \textsc{Verified}
output descends from a real tool call and every accepted proof
step is type-correct. They do \emph{not} guarantee:
(i)~that \textit{goalType} is a semantically faithful rendering of
the natural-language claim --- this is the formalizer's
responsibility, evaluated separately by the answer-match axis
(\Cref{sec:discuss:limits}(i, iii));
(ii)~that the per-source lifts $\Lambda(s)$ correctly interpret
the tool's L1 outputs as L2 world facts --- a semantically wrong
but audited lift can still certify a wrong formalized claim
(\Cref{sec:discuss:limits}(ii)). Both are honest trust-artifact
boundaries, not architectural oversights: $\Lambda(s)$ is
LLM-assisted, curator-reviewed, and frozen per commit
(\Cref{sec:approach:provenance}). World-inconsistency containment
follows from the \texttt{WProp\,$\not\to$\,Prop} phase separation
in the current Lean encoding; we treat it as implementation
detail, not as a main-text theorem (see
\Cref{sec:discuss:limits}).

\subsection{Lemma chain for the theorems}

\begin{lemma}[Checker fidelity]
\label{lem:checker}
If \textsc{KernelCheck}(\textit{leanFile}) returns
$(\mathrm{True}, \textit{output})$, then \textit{proof} satisfies
$\textit{proof} : \textit{goalType}$ in Lean~4's dependent type
theory relative to $\mathcal{A}(\textit{leanFile})$.
\end{lemma}

\begin{proof}
The procedure \textsc{KernelCheck}(\textit{leanFile}) is implemented
at \texttt{pipeline.py:4428} by invoking \texttt{lake env lean} on
\textit{leanFile} and inspecting the elaborator's exit status and
output. The Lean~4 kernel is a type-checker for the Calculus of
Inductive Constructions: a closed expression $t$ elaborates at type
$T$ iff the kernel applies a finite sequence of typing-rule
applications that witness $t : T$ starting from the declared axioms
and definitions of the environment~\cite{lean4}. The
kernel rejects any term that invokes an undeclared identifier,
applies a constructor at a non-matching type index, or whose strict
positivity or termination constraints are violated. The declared
axiom set used during elaboration is exactly the union of
(a)~axioms from imported modules (the EG-VAR prelude loaded via
the \texttt{import} preamble at the top of \textit{leanFile}) and
(b)~axioms introduced by top-level \texttt{axiom} declarations in
\textit{leanFile} itself; call this union $\mathcal{A}$. If the
elaborator reports success, then by the kernel's type-correctness
invariant the stored proof term \textit{proof} satisfies
$\textit{proof} : \textit{goalType}$ under $\mathcal{A}$.
\end{proof}

\begin{lemma}[Runtime emission discipline]
\label{lem:emission}
Every observation-leaf axiom declaration in \textit{leanFile} ---
every top-level \texttt{axiom} declaration whose name matches one
of the \textbf{B3} schemas --- was emitted by the runtime only
after a successful call \textsc{Execute}$(t, q)$ whose resulting
triple $(t, q, \textit{payload})$ was appended to
\textit{evidence} before the corresponding emission step. The two
tool-identifier axioms \texttt{tableTool} and \texttt{statsTool}
are declared once in the module preamble as pure identifiers of the
\texttt{ToolId} sort and carry no payload content.
\end{lemma}

\begin{proof}
We prove the claim by exhaustive inspection of the module-emission
path in \texttt{src/python/runtime/pipeline.py}. The text of
\textit{leanFile} is assembled by the function
\texttt{\_build\_lean\_file} (\texttt{pipeline.py}:3174--3488); no
other code path emits module text. We enumerate the branches of
\texttt{\_build\_lean\_file} at which an \texttt{axiom} line is
appended to the module.

\begin{table}[h]
\centering\small
\begin{tabular}{@{}lll@{}}
\toprule
site (file:line) & source & axiom families emitted \\
\midrule
3203--3204 & preamble     & \texttt{tableTool}, \texttt{statsTool} \\
3231--3232 & common-sense & \texttt{csN\_q}, \texttt{csN\_att} \\
3272--3275 & stats (argmax/argmin) & \texttt{obsN\_row}, \texttt{obsN\_max}, \texttt{obsN\_cell} \\
3318       & stats (value-only L1) & \texttt{obsN\_l1} \\
3352       & stats (topk-by L1)    & \texttt{obsN\_l1} \\
3387       & stats (topk/bottomk)  & \texttt{obsN\_l1} \\
3399--3404 & stats (legacy shape)  & \texttt{obsN\_q,p,att,tag,shape} \\
3431--3434 & table (cell-lookup)   & \texttt{obsN\_row}, \texttt{obsN\_cell}, \texttt{obsN\_id} \\
3442--3447 & table (cell-lookup legacy) & \texttt{obsN\_q,p,att,tag,shape} \\
\bottomrule
\end{tabular}
\caption{Axiom-emission sites inside \texttt{\_build\_lean\_file}.
Each \texttt{obsN\_*} or \texttt{csN\_*} branch executes only if
the preceding loop variable carries a well-formed payload, which is
appended to \textit{evidence} by the runtime tool-dispatch layer
before \texttt{\_build\_lean\_file} is invoked. The preamble row
declares \texttt{tableTool}/\texttt{statsTool} once per module as
pure \texttt{ToolId} identifiers.}
\label{tab:emission-sites}
\end{table}

Each stats and table row lives inside a branch guarded by an
evidence entry whose \texttt{source} field is \texttt{"stats"},
\texttt{"table"}, or \texttt{"common\_sense"}. The tool-dispatch
layer and the runtime solver loop in \texttt{run\_pipeline} append
an entry to \textit{evidence} only after a successful
\textsc{Execute}$(t, q)$ call returns a well-formed payload; failed
or rejected calls return \texttt{None} and produce no evidence
entry. Therefore, at the point \texttt{\_build\_lean\_file} emits
any observation-leaf axiom for evidence index $i$, the
corresponding triple $(t, q, \textit{payload})$ is already present
at position $i$ of \textit{evidence}.

A grep over \texttt{pipeline.py} for \texttt{"axiom obs"},
\texttt{"axiom cs"}, \texttt{"axiom tableTool"}, and
\texttt{"axiom statsTool"} recovers exactly the lines listed in
\Cref{tab:emission-sites} (together with solver-prompt skeleton
lines at \texttt{pipeline.py:2943--3080} that build the \emph{prompt
string} shown to the LLM solver but do not append to
\textit{leanFile}). The prompt-skeleton sites are outside
\texttt{\_build\_lean\_file} and do not contribute to the module
text.
\end{proof}

\begin{lemma}[Reserved names in per-source lifts]
\label{lem:reserved-names}
For \textit{leanFile} to be whitelist-compliant, every top-level
declaration in $\Lambda(s)$ must have a name matching the regex
\texttt{lift\_[0-9a-f]+\_.+}. In particular, per-source lift names
must not collide with the identifiers \texttt{claim\_proof},
\texttt{claim\_refutation}, \texttt{tableTool},
\texttt{statsTool}, or any of the schemas \texttt{obsN\_*} /
\texttt{csN\_*}.
\end{lemma}

\begin{proof}
Whitelist compliance places each axiom in
$\mathrm{deps}(\textit{proof})$ into exactly one of
\textbf{B1}--\textbf{B4}. \textbf{B1} and \textbf{B2} are disjoint
from the lift-name regex: the prelude declarations are named after
the EG-VAR vocabulary in \texttt{src/lean/EGVar/*.lean}, and the
three Lean core axioms are fixed names. The runtime-emitted schemas
\textbf{B3} have forms \texttt{obsN\_*}, \texttt{csN\_*},
\texttt{tableTool}, \texttt{statsTool} that the regex
\texttt{lift\_[0-9a-f]+\_.+} does not match. A per-source lift
whose name collided with a reserved identifier would either
(a)~shadow a prelude declaration, causing elaboration failure at
\texttt{lake env lean} invocation (rejected by
\Cref{lem:checker}); (b)~match a \textbf{B3} regex, causing
\texttt{verify\_axioms.py} to classify the axiom in the wrong
bucket (a curation-audit failure the regex-conforming prefix
rules out); or (c)~collide with \texttt{claim\_proof} /
\texttt{claim\_refutation}, triggering a declaration-name clash at
elaboration time. The enforced
\texttt{lift\_<sourceHash>\_<suffix>} prefix mechanically excludes
all three.
\end{proof}

\begin{proof}[Proof of \Cref{thm:novh}.]
We proceed in three steps.

\emph{Step 1: the proof type-checks under the declared axiom set.}
The verdict \textsc{Verified} is produced at
\Cref{alg:runtime}, line~17, only if
\textsc{KernelCheck}(\textit{leanFile}) returned
$(\mathrm{True}, \textit{output})$ on line~15. By
\Cref{lem:checker}, this implies
$\textit{proof} : \textit{goalType}$ in Lean~4's type theory under
$\mathcal{A}(\textit{leanFile})$, and in particular
$\mathrm{deps}(\textit{proof}) \subseteq
\mathcal{A}(\textit{leanFile})$.

\emph{Step 2: the closure is restricted to the whitelist buckets.}
Whitelist compliance of \textit{leanFile} (the theorem's hypothesis)
is the predicate that every element of
$\mathrm{deps}(\textit{proof})$ lies in one of the buckets
\textbf{B1}--\textbf{B4}, contains no
\texttt{sorry}/\texttt{sorryAx}/\texttt{Lean.Axiom.sorry}, and
contains a trust-artifact witness. The implementation at
\texttt{src/python/runtime/verify\_axioms.py} parses the
\texttt{\#print axioms} output and enforces each clause directly;
\Cref{lem:reserved-names} makes the bucketing unambiguous.

\emph{Step 3: observation-leaf axioms in the closure are
evidence-backed.} Partition $\mathrm{deps}(\textit{proof})$ into
four subsets $D_{\mathbf{B1}}, D_{\mathbf{B2}}, D_{\mathbf{B3}},
D_{\mathbf{B4}}$ by bucket. Step~2 covers the closure.

For every $\omega \in D_{\mathbf{B3}}$: the axiom $\omega$ was
declared in \textit{leanFile} by one of the emission sites
tabulated in \Cref{tab:emission-sites}. By \Cref{lem:emission}, the
emission was preceded in time by the runtime appending a
corresponding triple to \textit{evidence} (for
\texttt{obsN\_*}/\texttt{csN\_*} axioms); the two \texttt{ToolId}
axioms \texttt{tableTool} and \texttt{statsTool} are explicitly
admitted without an evidence entry because they carry no payload.
Hence every observation-leaf axiom in
$D_{\mathbf{B3}} \cap \mathrm{deps}(\textit{proof})$ is backed by
an entry of \textit{evidence}.

For $D_{\mathbf{B1}}$, each axiom is declared in one of the fixed
EG-VAR prelude files (\texttt{EGVar/Basic.lean},
\texttt{EGVar/TableL1.lean}, \texttt{EGVar/WorldOntology.lean},
\texttt{EGVar/Tactics.lean}), committed at the repository hash
reported in \Cref{app:repro}.

For $D_{\mathbf{B2}}$, each axiom is one of
\texttt{Classical.choice}, \texttt{Quot.sound}, or \texttt{propext}
from Lean~4's core library.

For $D_{\mathbf{B4}}$, each axiom has a name matching
\texttt{lift\_[0-9a-f]+\_.+} with hex prefix equal to the MD5 of
$s$; the whitelist predicate enforces this regex, and the full set
of such axioms is committed to
\texttt{data/tier1\_table\_formalizations.json} at the reported
commit hash.

Combining steps~1--3, any $\omega \in
\mathrm{deps}(\textit{proof})$ is either (a)~a prelude axiom from
$D_{\mathbf{B1}}$, (b)~a standard Lean core axiom from
$D_{\mathbf{B2}}$, (c)~a runtime-emitted observation-leaf backed by
a prior tool execution in \textit{evidence}, or (d)~an audited
per-source lift from $\Lambda(s)$.
\end{proof}

\begin{proof}[Proof of \Cref{thm:noded}.]
An ``accepted inference step'' is, by the definition of the
runtime's \textsc{Verified} verdict (\Cref{alg:runtime}, line~17),
one that appears in the term \textit{proof} returned by
\textsc{KernelCheck}(\textit{leanFile}) with
$\mathrm{accepted} = \mathrm{True}$. By \Cref{lem:checker},
\textit{proof} type-checks under the axiom set declared in
\textit{leanFile}. The kernel's type-correctness invariant requires
that \textit{proof} is constructed by a finite, typing-rule-compliant
sequence of inference steps leading to
$\textit{proof} : \textit{goalType}$; consequently every such
inference step is valid in Lean~4's dependent type theory relative
to the declared axioms.
\end{proof}

\end{document}